%% file: neurips_2022.tex
\documentclass{article}

 \PassOptionsToPackage{compress}{natbib}

\usepackage[final]{neurips_2022}

\usepackage[utf8]{inputenc} 
\usepackage[T1]{fontenc}    
\usepackage{hyperref}       
\usepackage{url}            
\usepackage{booktabs}       
\usepackage{amsfonts}       
\usepackage{nicefrac}       
\usepackage{microtype}      
\usepackage{xcolor}         
\usepackage{amsmath}
\usepackage{amssymb}
\usepackage{mathtools}
\usepackage{amsthm}

\input{icml_submission/header}

\usepackage[capitalize,noabbrev]{cleveref}

\theoremstyle{plain}

\theoremstyle{definition}

\theoremstyle{remark}

\usepackage[framemethod=TikZ]{mdframed}
\mdfdefinestyle{MyFrame}{%
    linecolor=black,
    outerlinewidth=2pt,
    roundcorner=20pt,
    innerrightmargin=20pt,
    innerleftmargin=20pt,
    backgroundcolor=white}
\usepackage{natbib}
\usepackage{hyperref}

\newcommand{\RR}{\mathbb{R}}

\newcommand{\noise}{\textsc{noise}}
\newcommand{\f}{\mathbf{f}}

\usepackage{tikz}
\usetikzlibrary{arrows}

\tikzset{
  treenode/.style = {align=center, inner sep=0pt, text centered,
    font=\sffamily},
  arn_n/.style = {treenode, circle, black, font=\sffamily\bfseries, draw=black,
    fill=white, text width=2.5em},
}
\title{Momentum Aggregation for Private Non-convex ERM}

\author{%
  Hoang Tran\\Boston University\\\texttt{tranhp@bu.edu}
  \And
  Ashok Cutkosky\\Boston University\\\texttt{ashok@cutkosky.com}
}

\begin{document}

\maketitle

\begin{abstract}\label{sec:abstrac}
  We introduce new algorithms and convergence guarantees for privacy-preserving non-convex Empirical Risk Minimization (ERM) on smooth $d$-dimensional objectives. We develop an improved sensitivity analysis of stochastic gradient descent on smooth objectives that exploits the recurrence of examples in different epochs. By combining this new approach with recent analysis of momentum with private aggregation techniques, we provide an $(\epsilon,\delta)$-differential private algorithm that finds a gradient of norm $\tilde O\left(\frac{d^{1/3}}{(\epsilon N)^{2/3}}\right)$ in $O\left(\frac{N^{7/3}\epsilon^{4/3}}{d^{2/3}}\right)$ gradient evaluations, improving the previous best gradient bound of $\tilde O\left(\frac{d^{1/4}}{\sqrt{\epsilon N}}\right)$.
\end{abstract}
\input{introduction}
\input{prelim}
\input{private_normalized}
\input{cross_epoch}
\input{conclusion}
\medskip

{\small
\bibliography{references}
\bibliographystyle{plainnat}
}

\section*{Checklist}

The checklist follows the references.  Please
read the checklist guidelines carefully for information on how to answer these
questions.  For each question, change the default \answerTODO{} to \answerYes{},
\answerNo{}, or \answerNA{}.  You are strongly encouraged to include a {\bf
justification to your answer}, either by referencing the appropriate section of
your paper or providing a brief inline description.  For example:
\begin{itemize}
  \item Did you include the license to the code and datasets? \answerYes{S}
  \item Did you include the license to the code and datasets? \answerNo{The code and the data are proprietary.}
  \item Did you include the license to the code and datasets? \answerNA{}
\end{itemize}
Please do not modify the questions and only use the provided macros for your
answers.  Note that the Checklist section does not count towards the page
limit.  In your paper, please delete this instructions block and only keep the
Checklist section heading above along with the questions/answers below.

\begin{enumerate}

\item For all authors...
\begin{enumerate}
  \item Do the main claims made in the abstract and introduction accurately reflect the paper's contributions and scope?
    \answerYes{}
  \item Did you describe the limitations of your work?
    \answerYes{}
  \item Did you discuss any potential negative societal impacts of your work?
    \answerNo{}{this is a theoretical paper. It has no potential negative societal impacts}
  \item Have you read the ethics review guidelines and ensured that your paper conforms to them?
    \answerYes{}
\end{enumerate}

\item If you are including theoretical results...
\begin{enumerate}
  \item Did you state the full set of assumptions of all theoretical results?
    \answerYes{}
        \item Did you include complete proofs of all theoretical results?
    \answerYes{}
\end{enumerate}

\item If you ran experiments...
\begin{enumerate}
  \item Did you include the code, data, and instructions needed to reproduce the main experimental results (either in the supplemental material or as a URL)?
    \answerNA{}
  \item Did you specify all the training details (e.g., data splits, hyperparameters, how they were chosen)?
    \answerNA{}
        \item Did you report error bars (e.g., with respect to the random seed after running experiments multiple times)?
    \answerNA{}
        \item Did you include the total amount of compute and the type of resources used (e.g., type of GPUs, internal cluster, or cloud provider)?
    \answerNA{}
\end{enumerate}

\item If you are using existing assets (e.g., code, data, models) or curating/releasing new assets...
\begin{enumerate}
  \item If your work uses existing assets, did you cite the creators?
    \answerNA{}
  \item Did you mention the license of the assets?
    \answerNA{}
  \item Did you include any new assets either in the supplemental material or as a URL?
    \answerNA{}
  \item Did you discuss whether and how consent was obtained from people whose data you're using/curating?
    \answerNA{}
  \item Did you discuss whether the data you are using/curating contains personally identifiable information or offensive content?
    \answerNA{}
\end{enumerate}

\item If you used crowdsourcing or conducted research with human subjects...
\begin{enumerate}
  \item Did you include the full text of instructions given to participants and screenshots, if applicable?
    \answerNA{}
  \item Did you describe any potential participant risks, with links to Institutional Review Board (IRB) approvals, if applicable?
    \answerNA{}
  \item Did you include the estimated hourly wage paid to participants and the total amount spent on participant compensation?
    \answerNA{}
\end{enumerate}

\end{enumerate}


\newpage
\appendix

\input{appendix}

\end{document}

%% file: icml_submission/header.tex
\usepackage{thmtools}
\usepackage{thm-restate}
\usepackage{algorithm}
\usepackage{algorithmic}
\usepackage{mathtools}
\usepackage{comment}

\declaretheorem[name=Theorem]{Theorem}
\declaretheorem[name=Proposition, numberlike=Theorem]{Proposition}
\declaretheorem[name=Lemma, numberlike=Theorem]{Lemma}
\declaretheorem[sibling=Theorem]{Definition}
\declaretheorem[sibling=Theorem]{Corollary}

\newcommand{\R}{\mathbb{R}}
\newcommand{\E}{\mathop{\mathbb{E}}}

\newcommand{\W}{\mathcal{W}}

\newcommand{\N}{\mathcal{N}}

\newcommand{\IN}{\textsc{IN}}
\newcommand{\OUT}{\textsc{OUT}}
\newcommand{\COMPOSE}{\textsc{COMPOSE}}

\newcommand{\Z}{\mathbb{Z}}

\newcommand{\var}{\textsc{Var}}

%% file: introduction.tex
\section{Introduction}\label{sec:intro}
In recent years, statistical machine learning models have been deployed in many domains such as health care, education, criminal justice, or social studies \citep{chen2021applications,article,Jiang230}. However, the release of statistical estimates based on these sensitive data comes with the risk of leaking personal information of individuals in the original dataset. One naive solution for this problem is to remove all the identifying information such as names, races, or social security numbers. Unfortunately, this is usually not enough to preserve privacy. It has been shown in various works that an adversary can take advantages of structural properties of the rest of the dataset to reconstruct information about certain individuals \citep{10.1145/1242572.1242598,10.1145/773153.773173}. Thus, we would need a stronger privacy-preserving mechanism. Over the past couple of decades, differential privacy \citep{10.1007/11681878_14} has emerged as the dominant privacy notion for machine learning problems.
\begin{Definition}
[Differential Privacy \citep{dwork2014algorithmic}] A randomized algorithm $M$: $\mathcal{X}^N \mapsto \RR^d$ satisfies $(\epsilon,\delta)-$ differential privacy (($\epsilon,\delta$)-DP) if for any two data sets $D,D' \in \mathcal{X}^N$ differing by at most one element and any event $E\subseteq \RR^d$, it holds that:
\begin{align*}
P\left[M(D) \in E\right] &\le \exp(\epsilon)P\left[M(D') \in E\right] + \delta
\end{align*}
\end{Definition}
Roughly speaking, differential privacy guarantees that the outputs of two neighboring datasets (datasets that differ in at most one datapoint) are \textit{almost} the same with high probability, thus preventing the adversary from identifying any individual's data.

In this paper, we are interested in designing $(\epsilon,\delta)-$private algorithms for non-convex empirical risk minimization (ERM) problems. In ERM problems, given $N$ i.i.d samples $x_1,...,x_N \in \mathcal{X}$ from some unknown distribution $P$, the goal is to find $w \in \R^d$ such that $w$ minimizes the empirical loss defined as follows:
\begin{align*}
    F(w) \triangleq \frac{1}{N}\sum_{i=1}^N f(w, x_i)
\end{align*}
where $f: \R^d \times \mathcal{X} \mapsto \R $ is the loss function associated with the learning problem. 
This is a setting that commonly arises in modern machine learning problem. For example, in image classification problems, the data point $x$ would be a tuple of (image, label), $w$ denotes the parameters of our model, and $f(w,x)$ represents composing the model predictions with some loss function such as cross-entropy. 
We are interested in finding a critical point, or a point such that the norm of the empirical gradient $\|\nabla F(w)\|$ is as small as possible. 
Further, we want all the outputs $w_1, w_2,..., w_T$ to be differentially private with respect to the $N$ training samples.

Private ERM has been well studied in the convex settings. The approaches in this line of work can be classified into three main categories: \textit{output perturbation} \citep{10.1007/11681878_14, chaudhuri2011differentially, zhang2017efficient, wu2017bolt},  \textit{objective perturbation} \citep{chaudhuri2011differentially, pmlr-v23-kifer12,8835258, talwar2014private}, and \textit{gradient perturbation} \citep{bassily2014private, wang2017differentially, NEURIPS2018_7221e5c8, wang2018differentially}. All of these approaches have been shown to achieve the asymptotically optimal bound $\tilde O\left(\frac{\sqrt{d}}{\epsilon N}\right)$ for smooth convex loss (with output perturbation requiring strong convexity to get the optimal bound) in (near) linear time. On the other hand, the literature on private non-convex ERM is nowhere as comprehensive. The first theoretical bound in private non-convex ERM is from \citep{zhang2017efficient}. They propose an algorithm called Random Round Private Stochastic Gradient Descent (RRSGD) which is inspired by the results from \citep{bassily2014private, ghadimi2013stochastic}. RRSGD is able to guarantee the utility bound of $ O\left(\frac{(d\log(n/\delta)\log(1/\delta))^{1/4}}{\sqrt{\epsilon N}}\right)$. However, RRSGD takes $O\left(N^2d\right)$ gradient computations to achieve this, which can be troublesome for high-dimensional problem. \citep{wang2018differentially} then improves upon this utility bound by a factor of $O\left((\log(n/\delta))^{1/4}\right)$. They achieve this rate by using full-batch gradient descent which is not a common practice in non-private machine learning in which very large batch sizes actually require careful work to make training efficient. Recently, \citep{wang2019efficient} tackles both runtime and utility issues by introducing a private version of the Stochastic Recursive Momentum (DP-SRM) \citep{cutkosky2019momentum}. By appealing to variance reduction as well as privacy amplification by subsampling \citep{balle2018privacy, abadi2016deep}, DP-SRM achieve the bound $O\left(\frac{(d\log(1/\delta))^{1/4}}{\sqrt{\epsilon N}}\right)$ in $O\left(\frac{(\epsilon N)^{3/2}}{d^{3/4} }+\frac{\epsilon N}{\sqrt{d}}\right)$ gradient complexity. However, DR-SRM still requires the batch size to be $O\left(\frac{\sqrt{\epsilon N}}{d^{1/4}}\right)$ for the analysis to work. Finally, although our focus in this paper is the ERM problem, there are also various works on private \emph{stochastic} non-convex/convex optimization \citep{bassily2019private, bassily2021differentially,bassily2021non,feldman2018privacy, feldman2020private, zhou2020private, asi2021private, wang2019differentially, kulkarni2021private}.

\textbf{Contributions.} We first provide the analysis for the private version of Normalized SGD (DP-NSGD) \citep{cutkosky2020momentum} for unconstrained non-convex ERM. By using the tree-aggregation technique \citep{chan2011private, dwork2010differential} to compute the momentum privately, we can ensure the privacy guarantee while adding noise of only $\tilde O\left(\frac{\sqrt{T}}{\epsilon\sqrt{N}}\right)$ (where $T$ is the total number of iterations). This allows us to achieve the same asymptotic bound $\tilde O\left(\frac{d^{1/4}}{\sqrt{\epsilon N}}\right)$ on the expectation of the gradient as \citep{zhang2017efficient, wang2018differentially, wang2019efficient} without appealing to privacy amplification techniques which is usually required for private SGD to have a good utility guarantee \citep{abadi2016deep}. DP-NSGD also does not require a large batch size as in \citep{wang2018differentially,wang2019efficient}; it has utility guarantee for any batch size.
This tree-aggregation technique is morally similar to the approach in \citep{kairouz2021practical, guha2013nearly} for online learning. However, unlike in \citep{kairouz2021practical, guha2013nearly}, we do not restrict our loss function to be convex and we will also extend tree-aggregation technique to SGD with momentum.
Further, we provide a new variant of Normalized SGD  that takes advantages of the fact that the gradients of the nearby iterates are close to each other due to smoothness. This new algorithm is able to guarantee an error of $\tilde O\left(\frac{d^{1/3}}{(\epsilon N)^{2/3}}\right)$ in $O\left(\frac{N^{7/3}\epsilon^{4/3}}{d^{2/3}}\right)$ gradient computations which, to our knowledge, is the best known rate for private non-convex ERM.

\textbf{Organization.} The rest of the paper is organized as follows. In section \ref{sec:preliminaries}, we define our problem of interest and the assumptions that we make on the problem settings. We also provide some background on Differential Privacy as well as some high-level intuition on tree-aggregation technique. We then formally describe our first private variant of Normalized SGD in section \ref{sec:1/3bound} and discuss its privacy guarantee and theoretical utility bound. In section \ref{sec:cross-epoch}, we introduce a novel sensitivity-reduced analysis for Normalized SGD that allows us to improve upon the utility bound in section \ref{sec:1/3bound}. Finally, we conclude with a discussion in section \ref{sec:conclusion}.

%% file: prelim.tex
\section{Preliminaries}\label{sec:preliminaries}
\textbf{Assumptions.} We define our loss function as $f(w,x): \RR^d \times \mathcal{X} \mapsto \RR$ where $x \in \mathcal{X}$ is some sample from the dataset. Throughout the paper, we make the following assumptions on $f(w,x)$. Define a differentiable function $f(.): \RR^d \mapsto \RR$ to be $G$-Lipschitz iff $\|\nabla f(y)\| \le G$ for all $y\in \RR^d$, and to be $L$-smooth iff $\|\nabla f(y_1) - \nabla f(y_2)\|\le L\|y_1 - y_2\|$ where $\|.\|$ is the standard 2-norm unless specified otherwise. We assume that $f(w,x)$ is differentiable, $G$-Lipschitz, and $L$-smooth with probability 1. We also assume that provided the initial point $w_1 \in \RR^d$ there exists some upper bound $R$ of $F(w_1)$ or formally: $\sup_{w\in \RR^d} F(w_1) - F(w) \le R$.

\textbf{Differential Privacy.} Let $D$ be a dataset containing $N$ datapoints. Then two datasets $D, D'$ are said to be neighbors if $\|D-D'\|_1 = 1$ or in other words, the two datasets differ at exactly one entry. Now we have the definition for sensitivity:
\begin{Definition}
($L_2$-sensitivity) The $L_2$-sensitivity of a function $f(.): \mathcal{X}^{N} \mapsto \RR^d$ is defined as follows:
\begin{align*}
    \Delta(f) &= \sup_{D,D'} \|f(D) - f(D')\|_2
\end{align*}
\end{Definition}
In this work, we also make use of Renyi Differential Privacy (RDP) \citep{mironov2017renyi}. RDP is a relaxation of $(\epsilon, \delta)-$DP and can be used to improve the utility bound and do composition efficiently while still guaranteeing privacy. 
\begin{Definition}
(($\alpha, \epsilon$)-RDP \citep{mironov2017renyi}) A randomized algorithm $f:  \mathcal{X}^{N} \mapsto \RR^d$ is said to have $\epsilon-$RDP of order $\alpha$ if for any neighboring datasets $D, D'$ it holds that:
\begin{align*}
    D_{\alpha}(f(D)||f(D')) \le \epsilon
\end{align*}
where $D_{\alpha}(P||Q) \triangleq \frac{1}{\alpha -1}\log \E_{x\sim  Q}\left(\frac{P(x)}{Q(x)}\right)^{\alpha}$
\end{Definition}
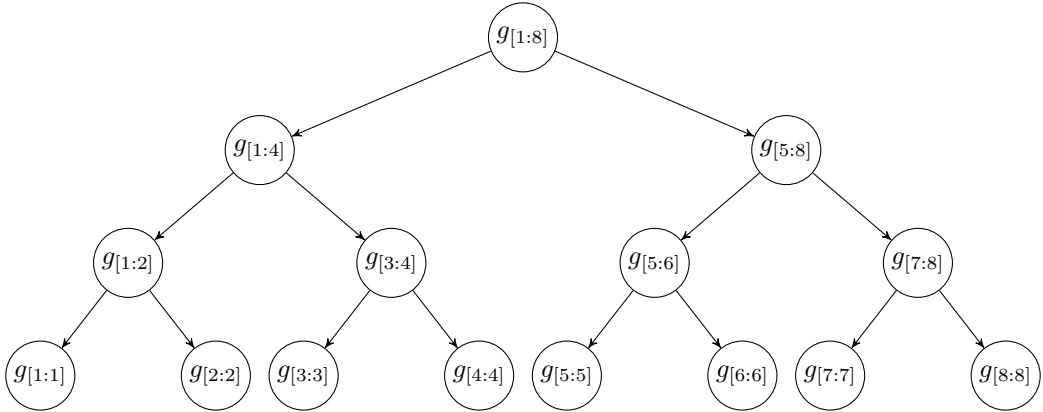
\begin{figure}[h]
    \centering
    \begin{tikzpicture}[->,>=stealth',level/.style={sibling distance = 7cm/#1,
  level distance = 1.5cm}] 
\node [arn_n] {$g_{[1:8]}$}
    child{ node [arn_n] {$g_{[1:4]}$} 
            child{ node [arn_n] {$g_{[1:2]}$} 
            	child{ node [arn_n] {$g_{[1:1]}$} 
                        } 
							child{ node [arn_n] {$g_{[2:2]}$}}
            }
            child{ node [arn_n] {$g_{[3:4]}$}
							child{ node [arn_n] {$g_{[3:3]}$}}
							child{ node [arn_n] {$g_{[4:4]}$}}
            }                            
    }
    child{ node [arn_n] {$g_{[5:8]}$}
            child{ node [arn_n] {$g_{[5:6]}$} 
							child{ node [arn_n] {$g_{[5:5]}$}}
							child{ node [arn_n] {$g_{[6:6]}$}}
            }
            child{ node [arn_n] {$g_{[7:8]}$}
							child{ node [arn_n] {$g_{[7:7]}$}}
							child{ node [arn_n] {$g_{[8:8]}$}}
            }
		}
; 
\end{tikzpicture}
    \caption{Visualization of the tree-aggregation mechanism. Each node of the tree hold the appropriate value $g_
{[y:z]}$ where $g_
{[y:z]} = \sum_{i=y}^zg_i$.}
    \label{fig:tree_figure}
\end{figure}
\textbf{Tree-aggregation.} Consider the problem of releasing private partial sum $s_{[y,z]} = \sum_{i=y}^z g_i$ given a stream of vectors $g_1,\dots,g_t$. Assuming $\|g_i\| \le G \ \forall i \in [t]$, then the simplest solution for this problem is to add Gaussian noise with standard deviation $\tilde O\left(1/\epsilon\right)$ to each $g_i$ so that each $g_i$ is $(\epsilon,\delta)-$DP. Then error from the noise added to any partial sum will be bounded by $O(\sqrt{T}/\epsilon)$. However, this problem can be solved much more efficiently using tree-aggregation \citep{chan2011private,dwork2010differential}. In tree-aggregation, we consider a complete binary tree where each of the leaf nodes is labeled with $g_i$ and each internal node in the tree is the sum of its children. The tree is illustrated in Fig.\ref{fig:tree_figure}. It is clear from Fig.\ref{fig:tree_figure} that each $g_i$ only affects its ancestors, or at most $\lceil\log t\rceil +1$ nodes. Thus, by adding Gaussian noise with standard deviation $O(\log t/\epsilon)$ to every node, the complete tree will be $(\epsilon,\delta)-$DP. Furthermore, since we only need at most $\log t$ nodes to compute any partial sum $s_{[y,z]}$, the noise added to every $s_{
[y,z]}$ would be bounded by $O(\log^{1.5}t/\epsilon)$, which is a lot less than the noise that we would add using advanced composition. 

To get a general idea on how we can use this tree-aggregation technique to design private algorithms, let us analyse the update of SGD with initial point $w_1 = 0$:
\begin{align*}
    w_{t+1} &= w_t - \eta \nabla f(w_t,x_t)\\
    &= -\eta \sum_{i=1}^t \nabla f(w_i,x_i)
\end{align*}
where $\nabla f(w_t,x_t)$ is the gradient of the loss function evaluated on the current iterate $w_t$ and some sample $x_t$. Notice that every iterate $w_{t+1}$ would be a function of the sum of the gradients $s_t = \sum_{i=1}^t\nabla f(w_{t},x_{t})$ up to iteration $t$. Thus, if we can use tree-aggregation to privately compute every partial sum $s_{[y,z]} = \sum_{i=y}^{z}\nabla f(w_i,x_i)$, we can the use these partial sums to compute every iterate $w_t$, and $w_t$ is also private by post-processing \citep{dwork2014algorithmic}. Similarly, we can also use this aggregation technique for SGD with momentum where we replace $\nabla f(w_{t},x_{t})$ with the momentum $m_t$. This method of computing momentum would be the main ingredient for our private algorithms. A more detailed discussions are provided in section \ref{sec:1/3bound} and \ref{sec:cross-epoch}. Tree-aggregated momentum has been briefly mentioned in the momentum variant of Differentially Private Follow-The-Regularized-Leader (DP-FTRL) in \citep{kairouz2021practical}. However, their algorithm does not directly compose the momentum using the tree-aggregation and they also do not provide any formal guarantee for the algorithm.

%% file: private_normalized.tex
\section{Private Normalized
SGD}\label{sec:1/3bound}
Before we go into the details of our private algorithm, first let us  set up some notations that are frequently used in the rest of the paper. We define the set of subsets $I$ as follows:
\begin{align*}
    I&=\{[a2^b+1,(a+1)2^b]|a,b\in \Z, [a2^b+1,(a+1)2^b]\subset\{1,\dots,T\}\}
\end{align*}
 Each interval in $I$ corresponds to a node in the tree described above (In Fig.1, the set $I$ is $\{[1,1], [2,2], [1,2],\dots\}$). Thus, the set of intervals $I$ is essentially a different way to store the binary tree for tree-aggregation mechanism. Then, for every interval $[a,b]$, there exists $O(\log T)$ disjoint subsets $[y_1,z_1],...,[y_n,z_n] \in I$ such that $[a,b] = \bigcup_{i=1}^n[y_i,z_i] $ \cite{daniely2015strongly}. We denote the set of such subsets as $\COMPOSE(a,b)$, which can be computed using Algorithm \ref{alg:compose}. Now we can proceed to the analysis of the private Normalized SGD algorithm described in Algorithm \ref{alg:normalized}.

The main idea of Algorithm \ref{alg:normalized} is based on two observations. The first observation is $w_t$ is a function of the momentum. Thus, similar to SGD, if we can compute the momentum privately, every iterate $w_1,\dots,w_T$ would also be private by post-processing. The second observation is that $m_t$ is an exponentially weighted sum of the gradients and can be computed efficiently using tree-aggregation. To see this, notice that the momentum $m_t$ can be written as follows:
\begin{align*}
    m_t &= (1-\alpha) m_{t-1} + \alpha \nabla f(w_t,x_{\pi^{q_t}_{r_t}})\\
    &=\alpha\sum_{i=1}^t(1-\alpha)^{t-i}\nabla f(w_i,x_{\pi^{q_i}_{r_i}})\\
    &= \alpha \sum_{[y,z] \in \COMPOSE(1,t)}(1-\alpha)^{t-z}\sum_{t' = y}^z(1-\alpha)^{z-{t'}}\nabla f(w_{t'},x_{\pi^{q_{t'}}_{r_{t'}}})
\end{align*}
where we denote $x_{\pi^{q_i}_{r_i}}$ as the sample drawn at iteration $i = q_iN + r_i$. Now, if we set $\f_{[y,z]} =\alpha\sum_{t' = y}^z(1-\alpha)^{z-{t'}}\nabla f(w_{t'},x_{\pi^{q_{t'}}_{r_{t'}}}) $, then:
\begin{align*}
    m_t &=  \sum_{[y,z] \in \COMPOSE(1,t)}(1-\alpha)^{t-z} \f_{[y,z]}
\end{align*}
Thus, if we consider $\f_{[y,z]}$ as the value of the node $[y,z]$ in the tree, we can compute $m_t$ using at most $O(\log t)$ nodes from the tree. We will also accumulate the noise in the same way. Assuming $\zeta_{[y,z]}$ as the noise we need to add to node $[y,z]$ to make $\f_{[y,z]}$ private, then the noisy $\hat m_t$ is:
\begin{align*}
    \hat m_t &= \sum_{[y,z] \in \COMPOSE(1,t)}(1-\alpha)^{t-z} \f_{[y,z]} +  \sum_{[y,z] \in \COMPOSE(1,t)}(1-\alpha)^{t-z} \zeta_{[y,z]}\\
    &= m_t +   \sum_{[y,z] \in \COMPOSE(1,t)}(1-\alpha)^{t-z} \zeta_{[y,z]}
\end{align*}
which is essentially the update of Algorithm \ref{alg:normalized}. Then, as long as $\f_{[y,z]}$ has low sensitivity, the noise we need to add to make $m_t$ private would be small since there are only at most $O\left(\log t\right)$ intervals in  $\COMPOSE(1,t)$.


\begin{algorithm}
   \caption{Differentially private Normalized SGD with momentum (DP-NSGD)}
   \label{alg:normalized}
   \begin{algorithmic}
     \STATE{\bfseries Input: } noise parameter $\sigma$, momentum parameter $\alpha = 1- \beta$, dataset $(x_1,\dots,x_{N})$, tree depth $R$, Lipschitz constant $G$, loss function $f(w,x)$, noise map Z.
      \STATE $m_0 \gets \hat m_0\gets 0$.
      \STATE $\noise \gets 0$
      \STATE $V\gets (\min(R, \lfloor\log_2(N)\rfloor)+1)\frac{T}{N}+ \sum_{j=\lfloor\log_2(N)+1\rfloor}^{R}\lfloor \frac{T}{2^j}\rfloor$.
      \FOR{$b = 0,\dots,R-1$}
      \FOR{$a = 0,\dots, \lfloor T/2^b \rfloor -1$}
      \STATE Sample $\zeta_{[a2^b+1,(a+1)2^b]} \sim N(0,16\alpha^2G^2\sigma^2VI)$.
      \STATE Z[$a2^b+1,(a+1)2^b$] $\gets \zeta_{[a2^b+1,(a+1)2^b]}$.
      \ENDFOR
      \ENDFOR
    \FOR{$q=0,\dots,\lfloor T/N \rfloor $}
    \STATE Sample a permutation $\pi_1^q,\dots, \pi^q_N$ of $\{1,\dots,N\}$ uniformly at random.
    \FOR{$i = 1,\dots, \min\{N, T-qN\}$}
    \STATE $t \gets qN+i$.
      \STATE Send interval $[1,t]$ to Algorithm \ref{alg:compose} to get the set $S = \COMPOSE(1,t)$.
      \FOR{every interval $S_i\in S$ }
     \STATE $\noise \gets \noise + (1-\alpha)^{t-(S_i[1])}Z[S_i]$ ( $S_i[1]$ indicates the end index of each interval $S_i$)
    \ENDFOR
      \STATE $m_t \gets (1-\alpha)m_{t-1}  + \alpha \nabla f(w_t, x^q_{\pi_i})$.
      \STATE $\hat m_t \gets m_t +\noise$
      \STATE $w_{t+1} \gets w_t - \eta\frac{\hat m_t}{\|\hat m_t\|}$
      \STATE $\noise \gets 0$
      \ENDFOR
      \ENDFOR
      \STATE\textbf{return} $w_1,\dots,w_T$
    \end{algorithmic}
\end{algorithm}
\begin{restatable}{Theorem}{momentumprivate}\label{thm:momentumprivacy}(Privacy guarantee)
Suppose that $f(w,x)$ is $G-$Lipschitz for all $w \in \RR^d, \ x\in \mathcal{X}$, Algorithm~\ref{alg:normalized} is $\left(z, \tfrac{z}{2\sigma^2}\right)$ Renyi-differentially private for all $z$. Consequentially, if $\delta\ge \exp(-\epsilon)$, then with $\sigma\ge \frac{2\sqrt{\log(1/\delta)}}{\epsilon}$, Algorithm~\ref{alg:normalized} is $(\epsilon,\delta)$-differentially private.
\end{restatable}
\begin{proof}[Proof sketch of Theorem \ref{thm:momentumprivacy}]
Denote $t = q_tN + r_t$ for any iteration t for $r_t \in [1,N]$ and $q_t \ge 0$. We define:
\begin{align*}
\f_{[y,z]} &=\alpha\sum_{t' = y}^z(1-\alpha)^{z-{t'}}\nabla f(w_{t'},x_{\pi^{q_{t'}}_{r_{t'}}}) \\
    m_t &= \sum_{[y,z]\in \COMPOSE(1,t)} (1-\alpha)^{t-z}\f_{[y,z]}
\end{align*}
To compute the sensitivity of $\f_{[y,z]}$, recall that the loss function $f$ is $G$-Lipschitz. Further, any given datapoint $x_i$ can contribute at most $\lceil |y-z|/N\rceil$ gradient terms in the summation defining $\f_{[y,z]}$. Thus $\f_{[y,z]}$ has sensitivity $\Delta = 2\alpha G\sum_{j=0}^{\lceil |y-z|/N\rceil-1}(1-\alpha)^{jN}\le 4\alpha G$ for $\alpha \ge \frac{1}{N}$ (Proposition \ref{prop:powersum}). Next, we compute the maximum number of interval that contains any datapoint $x_i$ for $i \in \{1,\dots,N\}$.  Given two neighboring datasets D and D', we use $s$ to indicate the index of the datapoint that is different between the two datasets. We then denote $S_{[y,z]} \subset \{1,...,N\}$ as the set that contains the indices of the datapoints that contribute to node $[y,z]$ (essentially which elements of the dataset are relevant to the value of node $[y,z]$). For any $i\in\{1,\dots, N\}$, for any $b\le \lfloor\log_2(N)\rfloor$, there are at most $\frac{T}{N}$ different $a\in \{0,\dots, \lfloor T/2^b\rfloor-1\}$ such that $s\in S_{[a2^b+1,(a+1)2^b]}$ for $b\le \lfloor\log_2(N)\rfloor$. Further, for any $b>\lfloor \log_2(N)\rfloor$, $i\in S_{[a2^b+1,(a+1)2^b]}=\{1,\dots, N\}$ for all $a$. Therefore, for any $i$ there are at most $V=\min(R+1, \lfloor\log_2(N)+1\rfloor)\frac{T}{N}+ \sum_{j=\lfloor\log_2(N)+1\rfloor}^{R}\lfloor \frac{T}{2^j}\rfloor$ sets $S_i$ such that $i\in S_i$. The privacy guarantee then follows via the ordinary composition properties of Renyi differential privacy. The full proof is included in the appendix.
\end{proof}
Before proving our utility bound, we need the following lemma.
\begin{restatable}{Lemma}{gradientbound}\label{lemma:gradientbound}[Essentially Lemma 2 [\cite{cutkosky2020momentum}]
Define:
\begin{align*}
    \hat \epsilon_t = \hat m_t - \nabla F(w_t)
\end{align*}
Suppose $w_1,...,w_T$ is a sequence of iterates defined by $w_{t+1} = w_t -\eta \frac{\hat m_t}{\|\hat m_t\|}$ for some arbitrary sequence $\hat m_1,..., \hat m_T$. Pick $\hat w$ uniformly at random from $w_1,...,w_T$. Then:
\begin{align*}
    \E\left[\|\nabla F(\hat w)\|\right]&\le \frac{3\E\left[(F(w_1) - F(w_{T+1}))\right]}{2\eta T} + \frac{3L\eta}{4} + \frac{3}{T}\sum_{t=1}^T\E[\|\hat \epsilon_t\|]
\end{align*}
\end{restatable}
Now we are ready to prove the utility guarantee of Algorithm \ref{alg:normalized}.
\begin{restatable}{Theorem}{normalizedbound}(Utility guarantee)
\label{thm:1/3bound}
Assuming $f(w,x)$ is $G$-Lipschitz, $L$-smooth for all $w\in \RR^d, \ x\in \mathcal{X}$, and $F(w_1)$ is bounded by $R$. Then Algorithm \ref{alg:normalized} with $\eta = \frac{1}{\sqrt{NT}}$, $\alpha = \frac{\epsilon N}{T\log_2 T\sqrt{d\log(1/\delta)}}$, $T = \frac{\epsilon N^2}{\log_2 T\sqrt{d\log(1/\delta)}}$, $\epsilon \le \frac{T\log_2T\sqrt{d\log(1/\delta)}}{ N}$, and $\hat w$ that is pick uniformly random from $w_1,\dots,w_T$ guarantees:
\begin{align*}
    \E\left[\|\nabla F(\hat w)\|\right]&\le \frac{(\frac{3}{2}R+6L\sqrt{\log_2T\log(1/\delta)} + 12G\sqrt{\log_2T}\log(1/\delta)^{1/4})d^{1/4}}{\sqrt{\epsilon N}} \\
    &\qquad+ \frac{3L(d\log(1/\delta))^{1/4}\sqrt{\log_2 T}}{4\sqrt{\epsilon}N^{3/2}} + \frac{6G}{\sqrt{N}} + \frac{6G\log_2 T\sqrt{d\log(1/\delta)}}{\epsilon N} \\
    &\le \tilde O\left(\frac{d^{1/4}}{\sqrt{\epsilon N}} + \frac{1}{\sqrt{N}}\right)
\end{align*}
\end{restatable}

The analysis of Algorithm \ref{alg:normalized} pretty much follows the analysis of non-private Normalized SGD in \cite{cutkosky2020momentum}. We first bound the error between the noisy momentum and the true empirical risk gradient $\E[\|\hat m_t -\nabla F(w_t)\|]\le  \E[\|\hat m_t - m_t\|] + \E[\|m_t - \nabla F(w_t)\|]$. The first term is bounded by the standard deviation of the added noise. The second term requires a more delicate analysis due to the use of random reshuffling (RR) in our algorithm (which is a more common practice than sampling in real-world problems). This introduces bias to the gradient estimates. Or formally, if we define $q_i = \lfloor \frac{i}{N}\rfloor$ and $r = i-q_iN$ for every iteration $i \in [T]$, we have $\E[\nabla f(w_i,x^{q_i}_{\pi_{r_i}})] \ne \nabla F(w_i)$. Most of the known analysis of RR is done with regular SGD \cite{mishchenko2020random, nguyen2019tight} since the bias becomes an even bigger problem in SGD with momentum due to the accumulative nature of momentum. To circumvent this problem we rewrite $m_t$ as $\alpha\sum_{i=1}^{t}(1-\alpha)^{t-i}\nabla f(w_{q_iN}, x^{q_i}_{\pi_{r_i}}) + O(\eta N)$ and use the fact that $\E[\nabla f(w_{q_i}, x_{\pi^{q_iN}_{r_i}})] = \nabla F(w_{q_iN})$ where $w_{q_i}$ is the iterate at the beginning of the epoch and $x^{q_i}_{\pi_{r_i}}$ is a datapoint that is sampled \textit{in} that particular epoch (this is true because $w_{q_iN}$ is independent of any data sampled in the epoch). Then, by applying Jensen's inequality and the definition of empirical loss, we have the bound $\E[\|\hat m_t -\nabla F(w_t)\|]\le \tilde O\left(\sqrt{\alpha} + \eta N + \frac{\eta}{\alpha} + \frac{\alpha\sqrt{dT}}{\epsilon \sqrt{N}}\right)$. Finally, we can apply Lemma \ref{lemma:gradientbound} with the appropriate hyperparameters to get the utility bound in Theorem \ref{thm:1/3bound}.

%% file: cross_epoch.tex
\section{Sensitivity Reduced Normalized SGD}\label{sec:cross-epoch}
\begin{algorithm}
   \caption{Sensitivity reduced normalized SGD}
   \label{alg:normalizedSGDacrossbatch}
   \begin{algorithmic}
      \STATE{\bfseries Input: } Initial Point $w_1$, training set $x_1,\dots,x_{N}$, learning rates $\eta$, momentum parameter $\alpha=1-\beta$, cross-batch parameter $\gamma$, Lipschitz constant $G$, smoothness constant $L$, noise variance $\sigma^2$.
      \STATE Set $V_{\le2 N}= 3\log_2 N, \ V_{>2N} = 4\frac{T}{N}\log_2(N) $.
      \STATE Set $\delta_G=4G , \ \delta_\Delta = \delta_r= 2\eta N$.
      \FOR{$q = 0,\dots, \lfloor T/N \rfloor $}
        \STATE Sample a permutation $\pi_1^q,\dots, \pi^q_N$ of $\{1,\dots,N\}$ uniformly at random.
        \FOR{$r = 0,\dots, \min\{N-1, T-qN\}$}
        \STATE $t \gets qN + r_t$
   \FOR{$k = 0, \dots, \lfloor \log_2(t)\rfloor$}
      \IF{$t \mod 2^k =0$}
      \IF{$t \le 2N$ }
       \STATE Sample $\zeta_t^G \sim N(0, \delta_G^2\sigma^2V_{\le 2N}I)$
      \ELSE
       \STATE Sample $\zeta_t^{G} \sim N(0, \delta_G^2\sigma^2V_{>2N}I)$
      \ENDIF
      \STATE Sample $\zeta_t^r \sim N(0, \delta_r^2\sigma^2V_{>2N}I)$.
      \STATE $\hat F^G_{[t-2^k+1,t]} = \sum_{i=t-2^k+1}^t(1-\alpha)^{t-i}\nabla f(w_{qN}, x_{\pi^{q}_{r_i}}) + \zeta_t^G $
      \STATE $\hat F^r_{[t-2^k+1,t]} = \sum_{i=t-2^k+1}^t(1-\alpha)^{t-i}(\nabla f(w_{i}, x_{\pi^{q}_{r_i}}) - \nabla f(w_{qN}, x_{\pi^{q}_{r_i}})) + \zeta_t^r $
      \IF{$t-2^k+1 > N$}
      \STATE Sample $\zeta_t^\Delta \sim N(0, \delta_\Delta^2\sigma^2V_{>2N}I) $.
      \STATE $\hat F^\Delta_{[t-2^k+1,t]} = \sum_{i=t-2^k+1}^t(1-\alpha)^{t-i}(\nabla f(w_{qN}, x_{\pi^{q}_{r_i}}) - \nabla f(w_{(q-1)N}, x_{\pi^{q}_{r_i}})) + \zeta_t^\Delta $
      \ENDIF
      \ENDIF
      \ENDFOR
    \STATE Compute $\hat m_t$ using the reconstruction algorithm (Algorithm \ref{alg:aggregate}).
      \STATE $w_t = w_{t-1} - \eta\frac{\hat m_t}{\|\hat m_t\|}$
      \ENDFOR
      \ENDFOR
      \STATE Return $w_1,..., w_T$
   \end{algorithmic}
\end{algorithm}
In this section, we will devise an algorithm that has an asymptotically tighter bound than $\tilde O\left(\frac{d^{1/4}}{\sqrt{\epsilon N}}\right)$, which is the best known bound for $(\epsilon, \delta)-$DP non-convex ERM. One natural direction is to use a better algorithm for non-convex optimization than SGD. It is known that in non-convex optimization operating on smooth losses, the optimal rate for SGD is $O(1/N^{1/4})$ \cite{arjevani2019lower} while more advanced methods that use variance reduction \cite{fang2018spider,zhou2018stochastic, tran2019hybrid} or Hessian-vectors product \cite{arjevani2020second, tran2021better} can achieve the rate $O(1/N^{1/3})$. Thus, one could hope that these methods can improve the bound of private non-convex ERM. Unfortunately, this does not seem to be the case in reality. For example, \cite{wang2019efficient} proposes a variant of the variance-reduction based stochastic recursive momentum (STORM) \cite{cutkosky2019momentum} called DP-SRM but their algorithm shows little asymptotic improvement compared to previous known bound. DP-SRM improves the error bound by a factor of $\log(n/\delta)$ but this seems to be the result of sampling \cite{abadi2016deep} rather than variance reduction. We would like to explore an alternative direction to improve the error bound. Instead of using a better non-convex optimization methods, we will reuse the Normalized SGD in section \ref{sec:1/3bound} but with a tighter privacy analysis that requires us to add less noise than $\tilde O\left(\frac{\sqrt{T}}{\epsilon \sqrt{N}}\right)$. 

The motivation for our methods comes from full-batch SGD. Consider the following full-batch normalized gradient descent  update: $w_{t+1} = w_t - \eta \frac{g_t}{\|g_t\|}$ where $g_t =\nabla F(w_t)$. Notice that (for $t>1$) $g_t= g_{t-1} + \nabla F(w_t) - \nabla F(w_{t-1})$. Now, given $g_{t-1}$, the sensitivity of $g_t$ can be  bounded by the sensitivity of $\nabla F(w_t)-\nabla F(w_{t-1})$, which in turn is only $O(\eta/N)$ due to smoothness and the fact that $\|w_t-w_{t-1}\|=\eta$. This is a lot smaller than the naive bound of $O(1/N)$ since $\eta \ll 1$. We can iterate  the above idea to write $g_t = \nabla F(w_1) + \sum_{i=2}^t \nabla F(w_i)-\nabla F(w_{i-1})$. In this way, each $g_t$ is a partial sum of  terms with sensitivity  $O(\eta/N)$, which can be be privately estimated with  an error of $\tilde O(\eta /N\epsilon)$ using tree-aggregation. The main intuition here is that except for the first iterate, each iterate is close to the previous iterate due to smoothness - that is, $\nabla F(w_t)-\nabla F(w_{t-1})$ has low sensitivity. In fact, we can further control the sensitivity of $g_t$ by expressing $g_t = (1-\gamma)(g_{t-1} + \nabla F(w_{t}) - \nabla F(w_{t-1})) + \gamma \nabla F(w_{t}) $ for some parameter $\gamma <1$. Now, with appropriate value of $\gamma$, $g_t$ can potentially have even lower sensitivity compared to the case $\gamma = 0$ that we discuss above. Moving beyond the full-batch case, we would like to replace the full-batch gradients $\nabla F(w_t)$ in this argument with momentum estimates $m_t$. Unfortunately, it is not the case that $m_t-m_{t-1}$ has low sensitivity. Instead, notice that the distance between two iterates evaluated on the same sample at two consecutive epochs will be at most $\eta N\ll 1$ for appropriate $\eta$, and we can show that the sensitivity of $m_t-m_{t-N}$ scales by this same distance factor. Our method described in Algorithm \ref{alg:normalizedSGDacrossbatch} follows this observation to compute low sensitivity momentum queries. To see this, let us rewrite the momentum $m_t $ in Algorithm \ref{alg:normalized}:
\begin{align}
    m_t &= (1-\alpha)m_{t-1}+\alpha \nabla f(w_t, x_{\pi^{q_t}_{r_t}})\nonumber\\
    &=\alpha\sum_{i=1}^{t}(1-\alpha)^{t-i}\nabla f(w_i, x_{\pi^{q_i}_{r_i}})\nonumber  \\
    &= \alpha\sum_{i=1}^{t}(1-\alpha)^{t-i}\nabla f(w_{q_iN}, x_{\pi^{q_i}_{r_i}}) + \alpha\sum_{i=1}^{t}(1-\alpha)^{t-i}(\nabla f(w_{i}, x_{\pi^{q_i}_{r_i}}) - \nabla f(w_{q_iN}, x_{\pi^{q_i}_{r_i}})) \label{eqn:momentumdecom}
\end{align}
In order to exploit this decomposition, we define $G_{[a,b]} = \sum_{t=a}^b (1-\alpha)^{b-t}\nabla f(w_{q_tN},x_{\pi^{q_t}_{r_t}})$, $r_{[a,b]} = \sum_{t=a}^b (1-\alpha)^{b-t}(\nabla f(w_t,x_{\pi^{q_t}_{r_t}})-\nabla f(w_{q_t N},x_{\pi^{q_t}_{r_t}}))$ and $\Delta_{[a,b]} = \sum_{t=a}^b (1-\alpha)^{b-t}(\nabla f(w_{q_t N},x_{\pi^{q_t}_{r_t}})-\nabla f(w_{(q_t-1) N},x_{\pi^{q_t}_{r_t}}))$. Then, (\ref{eqn:momentumdecom}) states:
\begin{align}
    m_t &= \alpha G_{[1,t]} + \alpha r_{[1,t]}\label{eqn:momentumdecomGr}
\end{align}
Considering $t = qN+r$, we can further refine this expansion by observing (Lemma \ref{lem:breaksum}):
\begin{align}
   G_{[1,t]} 
    &= (1-\alpha)^{t-r}G_{[1,r]}+\sum_{i=0}^{q-1}(1-\alpha)^{Ni}(1-\gamma)^{q-(i+1)}G_{[r+1,r+N]}\nonumber\\
    &\quad\quad+\sum_{i=1}^{q-1}\left(\sum_{j=0}^{i-1} (1-\gamma)^j(1-\alpha)^{(i-1-j)N}\right)\left((1-\gamma)\Delta_{[t-iN+1,t-(i-1)N]} + \gamma G_{[t-iN+1,t-(i-1)N]}\right)\label{eq:breaksum}
\end{align}
Thus, in Algorithm \ref{alg:normalizedSGDacrossbatch}, we will estimate $G_{[a,b]}$, $r_{[a,b]}$ and $\Delta_{[a,b]}$ for all intervals $[a,b]$ using tree aggregation, and then recombine these estimates using (\ref{eqn:momentumdecomGr}) and (\ref{eq:breaksum}) to estimate $m_t$. The critical observation is that each term of the sums for $r_{[a,b]}$ and $\Delta_{[a,b]}$ have sensitivity at most $O(\eta N)$, because the normalized update implies $\|w_t-w_{t'}\|\le \eta N$ for all $|t-t'|\le N$, while the $O(1)$ sensitivity terms in $G_{[a,b]}$ are scaled down by $\gamma$. However, notice that in the first two epochs, we don't have the extra parameter $\gamma$ that helps us control the noise added to the momentum. Thus, if we just naively add noise based on the maximum sensitivity (which is $V_{>2N}$),  the noise will blow up since it scales with $T/N$. To remedy this, we add two different types of noise to $G_{[t-2^k+1,t]}$ ($k \in \{0,\dots,\lfloor \log_2(t) \rfloor$) depends on whether $t\ge 2N$. By dividing the sensitivity into 2 cases, we can make the momentum of the first 2 epochs $(\epsilon,\delta)-$DP by adding noise of $\tilde O(1/\epsilon)$ and every other epochs $(\epsilon,\delta)-$DP by adding noise of $\tilde O(\gamma\sqrt{T}/\sqrt{N}\epsilon)$ . Then, the whole procedure would be $(O(\epsilon), O(\delta))-$DP overall.
Morally speaking, we can think of the first two terms of the update of Algorithm \ref{alg:normalizedSGDacrossbatch} (Eq. \ref{eq:breaksum}) as the first iterate in the full-batch example with some exponentially weighted factors and the third term as the $(1-\gamma)(\nabla F(w_{t}) -\nabla F(w_{t-1})) + \gamma \nabla F(w_{t}) $ part.
Overall, if we set $\gamma = \eta N$, the noise we need to add to make $m_t$ private is $\tilde O\left(\frac{\alpha}{\epsilon} + \frac{\alpha\sqrt{\eta T}}{\epsilon}+ \frac{\alpha\eta N\sqrt{T}}{\epsilon\sqrt{N}}\right)$, which is smaller than the $\tilde O(\frac{\alpha \sqrt{T}}{\epsilon\sqrt{N}})$ noise that we add in Section \ref{sec:1/3bound}. Formally, we have the following error bound for the momentum.

\begin{restatable}{Lemma}{crossepocherorr}\label{lemma:momentumaccuracy2/5bound}
Suppose $\alpha \ge \frac{1}{N}$ and $\gamma \le \frac{1}{2}$, using the update of Algorithm \ref{alg:normalizedSGDacrossbatch}, we have:
\begin{align*}
     \E[\|\hat m_t - m_t\|]&\le \frac{192G\alpha \log_2 T\sqrt{d\log(1/\delta)} }{\epsilon} + \frac{128(G+2L)\alpha \sqrt{\eta dT\log(1/\delta)}\log_2 T}{\epsilon} \\
     &\quad\quad+ \frac{16\eta L\alpha \log_2T\sqrt{3dTN\log(1/\delta)}}{\epsilon} 
\end{align*}
\end{restatable}
To prove Lemma \ref{lemma:momentumaccuracy2/5bound}, we sum up the variance of the Gaussian noise needed to add to make every term in Eq.\ref{eq:breaksum} and $\hat r_t$ private. Then $\E[\|\hat m_t - m_t\|] \le \sqrt{d\var}$ where $\var$ is the total noise variance and $d$ is the dimension of the noise vector.
\begin{restatable}
{Theorem}{crossprivacy}\label{thm:momentumacrossbatchesprivacy}(Privacy guarantee)
Suppose $\alpha \ge 1/N$ and $f$ is $G$-Lipschitz and $L$-smooth. Then Algorithm \ref{alg:normalizedSGDacrossbatch} is $(z,3z/2\sigma^2)$ Renyi differentially private for all $z$. With $\sigma \ge \frac{4 \sqrt{\log(1/\delta)}}{\epsilon}$, Algorithm~\ref{alg:normalizedSGDacrossbatch} is $(\epsilon,\delta)$-DP
\end{restatable}
The proof of Theorem \ref{thm:momentumacrossbatchesprivacy} follows the proof of tree-aggregated momentum with Renyi differential privacy as in section \ref{sec:1/3bound}. We defer the proof to the appendix.

Now we have the utility bound for Algorithm \ref{alg:normalizedSGDacrossbatch}.
\begin{restatable}{Theorem}{crossutility}\label{thm:crossutility}(Utility guarantee)
Assuming $f(w,x)$ is $G$-Lipschitz, $L$-smooth for all $w\in \RR^d, \ x\in \mathcal{X}$, and $F(w_1)$ is bounded by $R$. Then Algorithm \ref{alg:normalizedSGDacrossbatch} with $\gamma =\eta N, \eta = \frac{1}{\sqrt{NT}}$, $\alpha = \frac{N^{3/4}\epsilon}{T^{3/4}\sqrt{d}}$, $T = \frac{N^{7/3}\epsilon^{4/3}}{d^{2/3}}$, $\epsilon \le \frac{T^{3/4}\sqrt{d}}{N^{3/4}}$, and $\hat w$ that is picked uniformly at random from $w_1,...,w_T$ guarantees:
\begin{align*}
     \E[\|\nabla F(\hat w)\|] &\le \frac{(\frac{3}{2}R+24K(G+2L)+ 12L)d^{1/3}}{(\epsilon N)^{2/3}}+ \frac{6G}{\sqrt{N}} + \frac{36GK\sqrt{d}}{\epsilon N} + \frac{9KL\sqrt{d}}{\epsilon N}\\
    &\le \tilde O\left(\frac{d^{1/3}}{(\epsilon N)^{2/3}} + \frac{1}{\sqrt{N}}\right)
\end{align*}
where $K= 16\log_2 T\log(1/\delta)$.
\end{restatable}
The analysis of Algorithm \ref{alg:normalizedSGDacrossbatch} is similar to the analysis of Algorithm \ref{alg:normalized}. We still bound the error between the noisy momentum and the true empirical risk gradient as $\E\left[\|\hat m_t - \nabla F(w_t)\|\right] \le \E\left[\|\hat m_t - m_t\|\right] + \E\left[\|m_t -\nabla F(w_t)\|\right]$. $\E\left[\|m_t -\nabla F(w_t)\|\right]$ is the same  as in Theorem \ref{thm:1/3bound} but now, the noise bound $\E\left[\|\hat m_t - m_t\|\right]$ is only $\tilde O\left(\frac{\alpha\sqrt{d}}{\epsilon} + \frac{\alpha\sqrt{\eta dT}}{\epsilon} + \frac{\alpha\eta N\sqrt{dT}}{\epsilon\sqrt{N}}\right)$ instead of $\tilde O\left(\frac{\alpha\sqrt{dT}}{\epsilon\sqrt{N}}\right)$ due to our sensitivity-reduced analysis. Then we can apply Lemma \ref{lemma:gradientbound} to get our final utility guarantee. The full proof is provided in the appendix.

%% file: conclusion.tex
\section{Conclusions}\label{sec:conclusion}
In this paper, we present two new private algorithms for non-convex ERM. Both algorithms are variants of Normalized SGD with momentum and both utilizes tree-aggregation method to privately compute the momentum. Our first algorithm overcomes the large batch size and privacy amplification techniques requirements in previous  works while still achieving the state-of-the-art asymptotic bound $\tilde O\left(\frac{d^{1/4}}{\sqrt{\epsilon N}}\right)$. Our second algorithm uses the insight from full-batch SGD operating on smooth losses to reduce the sensitivity of the momentum. This allows the algorithm to achieve the utility bound $\tilde O\left(\frac{d^{1/3}}{(\epsilon N)^{2/3}}\right)$, which, to our knowledge, is the best known bound for private non-covex ERM.

\textbf{Limitations: } There are several limitations that one could further explore to improve the results of this paper. The most natural direction is to incorporate privacy amplification by shuffling into both algorithms. Based on the current analysis, it is not obvious that we can apply  privacy amplification immediately to tree-aggregated momentum. However, it is possible that there are some clever ways that one can come up to make the analysis work. It has been shown in previous work that privacy amplification by shuffling does help us achieve the $\tilde O\left(\frac{d^{1/4}}{\sqrt{\epsilon N}}\right)$ utility bound. Therefore, it is our hope that by combining privacy amplification with our sensitivity reduced algorithm, we can achieve an even better utility bound than $\tilde O\left(\frac{d^{1/3}}{(\epsilon N)^{2/3}}\right)$. Furthermore, even though our first algorithm has comparable run time to previous works, the second algorithm has slightly less ideal run time. The current algorithm requires $\tilde O\left(\frac{N^{7/3}\epsilon^{4/3}}{d^{2/3}}\right)$ gradient evaluations but has the total run time of $\tilde O(N^{11/3}/d^{2/3})$ due to the extra run time from the compose and reconstruction algorithms. We would like to reduce this to quadratic run time. We will leave these for future research.

%% file: appendix.tex
\section{Compose algorithm}
 The algorithm will return a list of intervals $S$ and for any $S_i \in S$, $S_i[0]$ is the start and $S_i[1]$ is the end of the respective interval. Intuitively, this $\COMPOSE$ function tells us which nodes in the tree that we need to use to compute any partial sum.
\begin{algorithm}
  \caption{COMPOSE $[a,b]$}
  \label{alg:compose}
  \begin{algorithmic}
      \STATE{\bfseries Input: } Starting point $a$, ending point $b$ of interval $[a,b]$.
      \STATE Let $k$ be the largest $k$ such that $a=1\mod 2^k$ and $a+2^k-1\le b$. 
      \STATE Set $S=\{[a, a+2^k-1]\}$
      \STATE Set $a'=a+2^k$.
      \IF{$a'>b$}
      \STATE \textbf{return} $S$.
      \ELSE
      \STATE Let $S'=\COMPOSE(a',b)$.
      \STATE Let $S=S\cup S'$.
      \STATE \textbf{return $S$}
      \ENDIF
  \end{algorithmic}
\end{algorithm}
\section{Reconstruction algorithm for sensitivity-reduced algorithm (Algorithm \ref{alg:normalizedSGDacrossbatch})}
\begin{algorithm}
\caption{Reconstruction algorithm}
\label{alg:aggregate}
   \begin{algorithmic}
       \STATE{\bfseries Input: } Noisy gradient arrays $\hat F^G, \hat F^\Delta, \hat F^r$, momentum parameter $\alpha$, sensitivity parameter $\gamma$, iteration $t$.
       \STATE $q \gets \lfloor \frac{t}{N}\rfloor$
       \STATE $r = t - qN$
        \STATE $\hat G_{[1,r]} = \sum_{[y,z] \in \COMPOSE(1,r)}(1-\alpha)^{r-z}\hat F^G_{[y,z]} $
        \STATE $\hat G_{[r+1,r+N]}= \sum_{[y,z] \in \COMPOSE(r+1,r+N)}(1-\alpha)^{r+N-z}\hat F^G_{[y,z]}$
        \STATE $\hat G_{[1,t]} = (1-\alpha)^{t-r}\hat G_{[1,r]}$ //initial value for $\hat G_{[1,t]}$ that will be updated in the following loop.
      \FOR{$i=0\dots q-1$}
      \STATE $\hat G_{[t-iN+1,t-(i-1)N]} = \sum_{[y,z] \in \COMPOSE(t-iN+1,t-(i-1)N)}(1-\alpha)^{t-(i-1)N-z}\hat F^G_{[y,z]}$
      \STATE $\hat \Delta_{[t-iN+1,t-(i-1)N]} = \sum_{[y,z] \in \COMPOSE(t-iN+1,t-(i-1)N)}(1-\alpha)^{t-(i-1)N-z}\hat F^\Delta_{[y,z]}$
      \STATE $\hat G_{[1,t]} += (\sum_{j=0}^{i-1}(1-\gamma)^j(1-\alpha)^{(i-1-j)N})((1-\gamma)\hat \Delta_{t-iN+1,t-(i-1)N} + \gamma \hat G_{t-iN+1,t-(i-1)N}) + (1-\alpha)^{Ni}(1-\gamma)^{q-i-1}\hat G_{[r+1,r+N]}$ 
      \ENDFOR
      \STATE $\hat r_t = \sum_{[y,z] \in \COMPOSE(1,t)}(1-\alpha)^{t-z}\hat F^r_{[y,z]}$
      \STATE $\hat m_t = \alpha \hat G_{[1,t]} + \alpha \hat r_t$
      \STATE{\bfseries Return: } $\hat m_t$
   \end{algorithmic}
\end{algorithm}
\section{Renyi Differential Privacy}
In this section, we will prove some general theorems on the composition of RDP using tree compositions.  Note that these results are all consequences of well-known properties of Renyi differential privacy \citep{mironov2017renyi} and tree aggregation \citep{dwork2010differential,chan2011private} and have been used in many other settings \citep{guha2013nearly, kairouz2021practical, asi2021private}. However, we find many presentations lacking in some detail, so we reproduce a complete description and analysi here for completeness.

We will consider functions operating on datasets of size $N$. Given two neighboring datasets $D=(Z_1,\dots,Z_N)$ and $D'=(Z'_1,\dots,Z'_N)$, we use $s$ to indicate the index that is different between the datasets. That is, $Z_i=Z'_i$ for $i\ne s$. Given a subset $S\subset \{1,\dots, N\}$, we use $D[S]$ to indicate the restriction of $D$ to the elements with index in $S$: $D[S]=(Z_i\ |\ i\in S)$.

Consider a set of $K$ functions $G_1,\dots,G_K$, and an associated set of \emph{subsets} of $\{1,\dots,N\}$, $S_1,\dots,S_K$ (note that the $S_i$ are sets of integers, NOT sets of datapoints). Each $G_i$ produces outputs in a space $\W$, and takes $i$ inputs: first, a dataset $D$ (or $D'$) of size $N$, and then $i-1$ elements of $\W$. That is, if the space of all datasets of size $N$ is $\mathcal{D}$, then $G_i:\mathcal{D}\times \W^{i-1}\to \W$. Further, each $G_i$ must have the property that $G_i$ depends \emph{only} on $D[S_i]$. that is, if $q\notin S_i$, then $G_i(D,x_1,\dots,x_{i-1})=
G_i(D',x_1,\dots,x_{i-1})$ for all $x_1,\dots,x_{i-1}$. With this in mind, we recursively define:
\begin{align*}
    \f_1&=G_1(D)\\
    \f'_1&=G_1(D')\\
    \f_i&=G_i(D,\f_1,\dots,\f_{i-1})\\
    \f'_i&=G_i(D',\f'_1,\dots,\f'_{i-1})
\end{align*}

Thus, the ordering of the $G_i$ indicates a kind of ``causality'' direction: later $\f_i$ are allowed to depend on the earlier $\f_i$, but the dependencies on the \emph{dataset} are fixed by the $S_i$. Intuitively, we should think of being able to ``break up'' a set of desired computational problems into a number of smaller computations such that (1) each sub-computation is in some way ``local'' in that it depends on only a small part of the dataset, potentially given the outputs of previous sub-computations, and (2) each of the original desired computational problems can be recovered from the answers to a small number of the sub-computations. If so, then we will be able to accurately and privately make all the desired computations.

Given the above background, we define: 
\begin{itemize}
    \item $\Delta_i$ to be the maximum sensitivity (with respect to some problem-specific metric like $L_2$ or $L_1$) over all $(x_1,\dots,x_{i-1})\in (\W)^{i-1}$ of the function $D\mapsto G_i(D,x_1,\dots,x_{i-1})$. That is, when $\W=\R^d$ and the metric is the one induced by a norm $\|\cdot\|$:
    \begin{align*}
        \Delta_i = \sup_{|D-D'|=1, x_1,\dots,x_{i-1}} \|G_i(D,x_1,\dots,x_{i-1})-G_i(D',x_1,\dots,x_{i-1})\|
    \end{align*}
    where we use $|\mathcal{D}-\mathcal{D}'|=1$ to indicate that datasets are neighboring.
    \item $\IN(s)$ to be the set of indices $i\in \{1,\dots,N\}$ such that $s\in S_i$.
    \item $\OUT(s)$ to be the complement of $\IN(s)$: the set of indices $i\in \{1,\dots,N\}$ such that $s\notin S_i$.
\end{itemize}

\subsection{Algorithm and Analysis}


Now, we will describe the aggregation algorithm that is used to compute private versions of all $f_t$, assuming that $\W=\R^d$ for some $d$. We write $X\sim \N(\mu,\sigma^2)$ to indicate that $X$ has density $p(X=x)=\frac{1}{\sigma \sqrt{2\pi}} \exp\left(-(x-\mu)^2/2\sigma^2\right)$. Further, for a vector $\mu$, we write $X\sim \N(\mu,\sigma^2 I)$ to indicate $P(X=x)=\frac{1}{(\sigma \sqrt{2\pi})^d} \exp\left(-\|x-\mu\|^2/2\sigma^2\right)$

Let $D_\alpha(P\| Q)$ indicate the Renyi divergence between $P$ and $Q$:
\begin{align*}
    D_\alpha(P\|Q)&=\frac{1}{\alpha-1}\log \E_{x\sim Q}\left[\left(\frac{P(x)}{Q(x)}\right)^\alpha\right]\\
    &=\frac{1}{\alpha-1}\log\left(\int_x Q(x)^{1-\alpha} P(x)^\alpha \ dx\right)
\end{align*}

Now, we need the following fact about Gaussian divergences:

\begin{align*}
    D_\alpha(\N(0,\sigma^2)\| \N(\mu,\sigma^2))=\alpha \mu^2/2\sigma^2
\end{align*}
This implies the following multi-dimensional extension:
\begin{align*}
    D_\alpha(\N(0,\sigma^2 I)\|\N(\mu,\sigma^2 I))&=\frac{1}{\alpha-1}\log\left(\frac{1}{(\sigma \sqrt{2\pi})^d}\int_{\R^d} \exp(-\|x-\mu\|^2/2\sigma^2)^{1-\alpha}\exp(-\|x\|^2/2\sigma^2)^\alpha)\ dx_1\dots dx_d\right)\\
    &\frac{1}{\alpha-1}\left[\sum_{i=1}^d \log\left(\frac{1}{\sigma\sqrt{2\pi}}\int_{-\infty}^\infty \exp(-(x-\mu_i)^2/2\sigma^2)^{1-\alpha}\exp(-x^2/2\sigma^2)^\alpha) dx\right) \right]\\
    &=\sum_{i=1}^d D_\alpha(\N(0,\sigma^2)\| \N(\mu_i,\sigma^2))\\
    &=\alpha\|\mu\|^2/2\sigma^2
\end{align*}

\begin{algorithm}
   \caption{Aggregation Algorithm with Gaussian noise}
   \label{alg:gaussianaggregation}
   \begin{algorithmic}
      \STATE{\bfseries Input: } Dataset $D$, functions $G_1,\dots,G_K$ with sensitivities $\Delta_1,\dots,\Delta_K$ with respect to the $L_2$ norm. Noise parameter $\rho$
      \STATE Sample random $\zeta_1\sim \N(0,\Delta_1^2/\rho^2 I)$
      \STATE Set $\f_1 = G_1(D)$.
      \STATE Set $\hat \f_1= \f_1 + \zeta_1$
      \FOR{$i=2,\dots,K$}
      \STATE Sample random $\zeta_t\sim \N(0,\Delta_i^2/\rho^2 I)$.
      \STATE Set $\f_i =G_i(D, \hat \f_1,\dots,\hat \f_{i-1})$.
      \STATE Set $\hat \f_i= \f_i+ \zeta_i$.
      \ENDFOR
      \STATE \textbf{return} $\hat \f_1,\dots,\hat \f_K$.
   \end{algorithmic}
\end{algorithm}

\begin{restatable}{Theorem}{thmrenyiaggregation}\label{thm:renyiaggregation}
Let $V$ to be the maximum over all $s\in\{1,\dots, N\}$ of the total number of sets $S_i$ such that $s\in S_i$ (i.e. $V=\sup_s |\IN(s)|$). Then Algorithm \ref{alg:gaussianaggregation} is $(\alpha, V \alpha\rho^2/2)$ Renyi differentially private for all $\alpha$.
\end{restatable}

To convert the above result to ordinary differential privacy, we observe that $(\alpha,\epsilon)$-RDP implies $(\epsilon + \frac{\log(1/\delta)}{\alpha-1},\delta)$-DP for all $\delta$. Thus, supposing $\rho\le \sqrt{\frac{\log(1/\delta)}{V}}$, we then set $\alpha=1+\frac{\sqrt{\log(1/\delta)}}{\rho\sqrt{V}}$ to get $(V\rho^2 + \rho\sqrt{V\log(1/\delta)},\delta)\le (2\rho\sqrt{V\log(1/\delta)},\delta)$ differential privacy.
\begin{proof}
Let us write $\hat \f_i$ for the outputs with input dataset $D$, and $\hat \f'_i$ for the outputs with input dataset $D'$. Let $s$ be the index such that $Z_q\ne Z'_q$.

Then, we can express the joint density of the random variable $\hat \f_1,\dots,\hat \f_K$:
\begin{align*}
    p(\hat \f_1=r_1,\dots,\hat \f_k=r_k)&=\prod_{i=1}^Kp(\hat \f_i=r_i|\hat \f_1=r_1,\dots,\hat \f_{i-1}=r_{i-1})\\
    &=\prod_{i\in \IN(s)}p(\hat \f_i=r_i|\hat \f_1=r_1,\dots,\hat \f_{i-1}=r_{i-1})\prod_{i\in \OUT(s)}p(\hat \f_i=r_i|\hat \f_1=r_1,\dots,\hat \f_{i-1}=r_{i-1})
\end{align*}
Similar expressions hold for $\hat \f'_i$:
\begin{align*}
    p(\hat \f'_1=r_1,\dots,\hat \f'_k=r_k)&=\prod_{i\in \IN(s)}p(\hat \f'_i=r_i|\hat \f'_1=r_1,\dots,\hat \f'_{i-1}=r_{i-1})\prod_{i\in \OUT(s)}p(\hat \f'_i=r_i|\hat \f'_1=r_1,\dots,\hat \f'_{i-1}=r_{i-1})
\end{align*}

Further, for any $i\in \OUT(s)$
\begin{align*}
    p(\hat \f_i=r_i|\hat \f_1=r_1,\dots,\hat \f_{i-1}=r_{i-1})&=p(\hat \f'_i=r_i|\hat \f'_1=r_1,\dots,\hat \f'_{i-1}=r_{i-1})
\end{align*}

Now, for the rest of the proof, we mimic the proof of composition for Renyi differential privacy: let $P$ and  $P'$ be the distributions of the ouputs under $D$ and $D'$. Then:

\begin{align*}
    D_\alpha(P'||P)&=\frac{1}{\alpha-1}\log\left[\int_{r}\prod_{i\in \IN(s)}p(\hat \f_i=r_i|\hat \f_1=r_1,\dots,\hat \f_{i-1}=r_{i-1})^{1-\alpha}\right.\\
    &\quad\quad\quad\quad\prod_{i\in \OUT(s)}p(\hat \f_i=r_i|\hat \f_1=r_1,\dots,\hat \f_{i-1}=r_{i-1})^{1-\alpha}\\
    &\quad\quad\quad\quad\left.\prod_{i\in \IN(s)}p(\hat \f'_i=r_i|\hat \f'_1=r_1,\dots,\hat \f'_{i-1}=r_{i-1})^\alpha\prod_{i\in \OUT(s)}p(\hat \f'_i=r_i|\hat \f'_1=r_1,\dots,\hat \f'_{i-1}=r_{i-1})^\alpha\ dr\right]\\
    &=\frac{1}{\alpha-1}\log\left[\int_{r}\prod_{i\in \IN(s)}p(\hat \f_i=r_i|\hat\f_1=r_1,\dots,\hat \f_{i-1}=r_{i-1})^{1-\alpha}p(\hat \f'_i=r_i|\hat \f'_1=r_1,\dots,\hat \f'_{i-1}=r_{i-1})^\alpha\right.\\
    &\left.\quad\quad\quad\quad\prod_{i\in \OUT(s)}p(\hat \f_i=r_i|\hat \f_1=r_1,\dots,\hat \f_{i-1}=r_{i-1})^{1-\alpha}p(\hat \f'_i=r_i|\hat \f'_1=r_1,\dots,\hat \f'_{i-1}=r_{i-1})^\alpha\ dr\right]
\end{align*}
To make notation a little more precise, let us write $dr_{\IN}$ to indicate $\wedge_{i\in \IN(s)} dr_i$ and $dr_{\OUT}=\wedge_{i\in \OUT(s)} dr_i$. Similarly, $\int_{r_{IN}}$ indicates integration over only $r_i$ such that $i\in \IN(s)$.
Now, recall that for $i\in \OUT(s)$, we have 
\begin{align*}
    p(\hat \f_i=r_i|\hat \f_1=r_1,\dots,\hat \f_{i-1}=r_{i-1})=p(\hat \f'_i=r_i|\hat \f'_1=r_1,\dots,\hat \f'_{i-1}=r_{i-1})
\end{align*}
so that:
\begin{align*}
    D_\alpha(P'||P)&=\frac{1}{\alpha-1}\log\left[\int_{r}\prod_{i\in \IN(s)}p(\hat \f_i=r_i|\hat \f_1=r_1,\dots,\hat \f_{i-1}=r_{i-1})^{1-\alpha}p(\hat \f'_i=r_i|\hat \f'_1=r_1,\dots,\hat \f'_{i-1}=r_{i-1})^\alpha\right.\\
    &\left.\quad\quad\quad\quad\prod_{i\in \OUT(s)}p(\hat \f_i=r_i|\hat \f_1=r_1,\dots,\hat \f_{i-1}=r_{i-1})\ dr_{\OUT}dr_{\IN}\right]
    \intertext{so, we can integrate over $r_i$ for $i\in \OUT(s)$:}
    &=\frac{1}{\alpha-1}\log\left[\int_{r_{\IN}}\prod_{i\in \IN(s)}p(\hat \f_i=r_i|\hat \f_1=r_1,\dots,\hat \f_{i-1}=r_{i-1})^{1-\alpha}p(\hat \f'_i=r_i|\hat \f'_1=r_1,\dots,\hat \f'_{i-1}=r_{i-1})^\alpha\ dr_{\IN}\right]\\
\end{align*}
Now, let the indices in $\IN(s)$ be (in order): $i_1,\dots,i_n$. Then we have:
\begin{align*}
    D_\alpha&(P'||P)\\
    &=\frac{1}{\alpha-1}\log\left[\left(\prod_{j=1}^n\int_{r_{i_j}}p(\hat \f_{i_j}=r_{i_j}|\hat \f_1=r_1,\dots,\hat \f_{i_j-1}=r_{i_j-1})^{1-\alpha}p(\hat \f'_{i_j}=r_{i_j}|\hat \f'_1=r_1,\dots,\hat \f'_{i_j-1}=r_{i_j-1})^\alpha\right)
    dr_{\IN}\right]
    \intertext{Now, let's compress the density notation a bit to save space:}
    &=\frac{1}{\alpha-1}\log\left[\left(\prod_{j=1}^n\int_{r_{i_j}}p(\hat \f_{i_j}=r_{i_j}|r_1,\dots,r_{i_j-1})^{1-\alpha}p(\hat \f'_{i_j}=r_{i_j}|r_1,\dots,r_{i_j-1})^\alpha\right)dr_{i_1}\dots dr_{i_n}\right]
\end{align*}
Now, let's focus on just one integral:
\begin{align*}
&\int_{r_{i_j}}p(\hat \f_{i_j}=r_{i_j}|r_1,\dots,r_{i_j-1})^{1-\alpha}p(\hat \f'_{i_j}=r_{i_j}|r_1,\dots,r_{i_j-1})^\alpha dr_{i_j} \\
&=\int_{r_{i_j}}p( \zeta_{i_j}=r_{i_j}-G_{i_j}(D,r_1,\dots,r_{i_j-1}))^{1-\alpha}p(\zeta'_{i_j}=r_{i_j}-G_{i_j}(D',r_1,\dots,r_{i_j-1}))^\alpha\ dr_{i_j}\\
&=\frac{1}{(\sigma \sqrt{2\pi})^d}\int_{\R^d}\exp\left(-\frac{(1-\alpha)\|x-G_{i_j}(D,r_1,\dots,r_{i_j}-1)\|^2}{\Delta_{i_j}^2/\rho^2}\right)\exp\left(-\frac{\alpha\|x-G_{i_j}(D',r_1,\dots,r_{i_j}-1)\|^2}{\Delta_{i_j}^2/\rho^2}\right)\ dx
\intertext{using a change of variables $z=x-G_{i_j}(D',r_1,\dots,r_{i_j}-1)$:}
&=\frac{1}{(\sigma \sqrt{2\pi})^d}\int_{\R^d}\exp\left(-\frac{(1-\alpha)\|z-(G_{i_j}(D,r_1,\dots,r_{i_j}-1) -G_{i_j}(D',r_1,\dots,r_{i_j}-1))\|^2}{\Delta_{i_j}^2/\rho^2}\right)\exp\left(-\frac{\alpha\|z\|^2}{\Delta_{i_j}^2/\rho^2}\right)\ dz\\
&=\exp((\alpha-1)D_\alpha(\N(0,\Delta_{i_j}^2/\rho^2), \N(G_{i_j}(D,r_1,\dots,r_{i_j}-1) -G_{i_j}(D',r_1,\dots,r_{i_j}-1),\Delta_{i_j}^2/\rho^2)))
\intertext{use our expression for divergence between Gaussians with the same covariance:}
&= \exp\left((\alpha-1)\alpha \rho^2\|G_{i_j}(D,r_1,\dots,r_{i_j}-1) -G_{i_j}(D',r_1,\dots,r_{i_j}-1)\|^2/2\Delta_{i_j}^2\right)\\
&\le \exp\left((\alpha-1) \alpha\rho^2/2\right)
\end{align*}

Now, returning to our bound on the divergence:
\begin{align*}
    D_\alpha(P'||P)&=\frac{1}{\alpha-1}\log\left[\left(\prod_{j=1}^n\int_{r_{i_j}}p(\hat \f_{i_j}=r_{i_j}|r_1,\dots,r_{i_j-1})^{1-\alpha}p(\hat \f'_{i_j}=r_{i_j}|r_1,\dots,r_{i_j-1})^\alpha\right)dr_{i_1}\dots dr_{i_n}\right]
    \intertext{rewrite a bit for clarity:}
    &=\frac{1}{\alpha-1}\log\left[\left(\prod_{j=1}^{n-1}\int_{r_{i_j}}p(\hat \f_{i_j}=r_{i_j}|r_1,\dots,r_{i_j-1})^{1-\alpha}p(\hat \f'_{i_j}=r_{i_j}|r_1,\dots,r_{i_j-1})^\alpha\right.\right.\\
    &\left.\left.\qquad\qquad\qquad\qquad\qquad \int_{r_{i_n}}p(\hat \f_{i_n}=r_{i_n}|r_1,\dots,r_{i_n-1})^{1-\alpha}p(\hat \f'_{i_n}=r_{i_n}|r_1,\dots,r_{i_n-1})^\alpha \right)dr_{i_1}\dots dr_{i_n}\right]
    \intertext{integrate out $r_{i_n}$:}
    &=\frac{1}{\alpha-1}\log\left[\exp\left((\alpha-1)\alpha\rho^2/2\right)\left(\prod_{j=1}^{n-1}\int_{r_{i_j}}p(\hat \f_{i_j}=r_{i_j}|r_1,\dots,r_{i_j-1})^{1-\alpha}p(\hat \f'_{i_j}=r_{i_j}|r_1,\dots,r_{i_j-1})^\alpha\right)\right.\\
    &\qquad\qquad\qquad\left.dr_{i_1}\dots dr_{i_{n-1}}\right]
    \intertext{now integrate out all the other variables one by one:}
    &\le \frac{1}{\alpha-1}\log\left[\prod_{j=1}^n \exp\left((\alpha-1)\alpha\rho^2/2\right)\right]\\
    &=n\alpha\rho^2/2\\
    &\le V\alpha\rho^2/2
\end{align*}
\end{proof}
\section{Proof of section \ref{sec:1/3bound}}
\subsection{Privacy}
\momentumprivate*
To see the $(\epsilon,\delta)$-DP result from the RDP bound, we observe that $(z,\tfrac{z}{2\sigma^2})$-RDP implies $(\tfrac{z}{2\sigma^2} + \frac{\log(1/\delta)}{z-1},\delta)$-DP for all $\delta$. Thus, optimizing over $z$, we set $z=1+\sqrt{2\sigma^2\log(1/\delta)}$ to obtain $(\epsilon,\delta)$-DP with $\epsilon = 1/2\sigma^2 + 2\sqrt{\log(1/\delta)/2\sigma^2}$. Thus, by quadratic formula, we ensure $(\epsilon,\delta)$-DP for all $\sigma$ satisfying:
\begin{align*}
    \sigma \ge \frac{1}{\sqrt{2\log(1/\delta)+2\epsilon}-\sqrt{2\log(1/\delta)}}
\end{align*}

In pursuit of a simpler expression, observe that $\sqrt{x+y}\ge \sqrt{x} + \frac{y}{2\sqrt{x+y}}$, so that it suffices to choose:
\begin{align*}
    \sigma &\ge \frac{\sqrt{2\log(1/\delta)+2\epsilon}}{\epsilon}
\end{align*}
So, in particular if $\delta \ge \exp(-\epsilon)$, then we obtain the expression in the Theorem statement.
\begin{proof}
Let $t = q_tN+ r_t$. Define:
\begin{align*}
        \f_{[y,z]}&=\alpha \sum_{t=y}^z(1-\alpha)^{z-t}\nabla f(w_t, x_{\pi^{q_t}_{\pi_{r_t}}})
\end{align*}
To compute the sensitivity of $\f_{[y,z]}$, recall that $\|\nabla f(w_t, x_{\pi^{q_t}_{\pi_{r_t}}})\|\le G$ for all $w$ and $x$. Further, any given $x_i$ can contribute at most $\lceil |y-z|/N\rceil$ gradient terms in the summation defining $\f_{[y,z]}$, and the $j$th such term is scaled by $\alpha(1-\alpha)^{jN}$. Thus, $\f_{[y,z]}$ has sensitivity $\Delta_{[y,z]}=2\alpha G\sum_{j=0}^{\lceil |z-y|/N\rceil -1}(1-\alpha)^{jN} \le 4\alpha G $ (by proposition \ref{prop:powersum}) for all $[y,z]$ and $\alpha \ge \frac{1}{N}$.
 
Next we compute the maximum value over all $s\in\{1,\dots,N\}$ of $|\{S_i|\ s\in S_i\}|$ (the maximum number of intervals that one index s can belong to). For any $s\in\{1,\dots, N\}$, for any $b\le \lfloor\log_2(N)\rfloor$, there are at most $\frac{T}{N}$ different $a\in \{0,\dots, \lfloor T/2^b\rfloor-1\}$ such that $s\in S_{[a2^b+1,(a+1)2^b]}$ for $b\le \lfloor\log_2(N)\rfloor$. Further, for any $b>\lfloor \log_2(N)\rfloor$, $s\in S_{[a2^b+1,(a+1)2^b]}=\{1,\dots, N\}$ for all $a$. Therefore, for any $s$ there are at most $V=\min(R+1, \lfloor\log_2(N)+1\rfloor)\frac{T}{N}+ \sum_{j=\lfloor\log_2(N)+1\rfloor}^{R}\lfloor \frac{T}{2^j}\rfloor$ sets $S_i$ such that $s\in S_i$. 

Now, we show that Algorithm~\ref{alg:normalized} is actually providing output distributed in the same way as the aggregation mechanism in Algorithm~\ref{alg:gaussianaggregation}. To do this, observe that given values for $w_1,\dots,w_T$, $\hat \f_{[y,z]} = \zeta_{[y,z]}+\alpha \sum_{t=y}^z(1-\alpha)^{z-t}\nabla f(w_t, x_{\pi^{q_t}_{\pi_{r_t}}})$ where $\zeta_{[y,z]}\sim\N(0,\Delta_{[y,z]}^2\sigma^2V I)$ using the notation of Algorithm~\ref{alg:gaussianaggregation}. Then the output of Algorithm~\ref{alg:gaussianaggregation} is
\begin{align}
     \hat G_{[1,t]}(\hat \f_{[1,1]},\dots,\hat \f_{[1,t]}) &= \sum_{[y,z]\in \COMPOSE(1,t)}(1-\alpha)^{t-z}\hat \f_{[y,z]}\nonumber\\
     &=\alpha \sum_{[y,z]\in \COMPOSE(1,t)}(1-\alpha)^{t-z}\sum_{t'=y}^z(1-\alpha)^{z-t'}\nabla f(w_{t'}, x_{\pi^{q_{t'}}_{\pi_{r_{t'}}}})\nonumber \\
     &\qquad\qquad+ \sum_{[y,z]\in \COMPOSE(1,t)} (1-\alpha)^{t-z}\zeta_{[y,z]}\nonumber\\
     &=\alpha \sum_{t'=1}^t(1-\alpha)^{t-t'} \nabla f(w_{t'}, x_{\pi^{q_{t'}}_{\pi_{r_{t'}}}}) \nonumber\\
     &\qquad\qquad+ \sum_{[y,z]\in \COMPOSE(1,t)} (1-\alpha)^{t-z}\zeta_{[y,z]}\label{eqn:renyinoise}
\end{align}
Now, observe that the value of $m_t$ as described in Algorithm~\ref{alg:normalized} can be written as:
\begin{align*}
    m_t &= (1-\alpha)m_{t-1}+\alpha  \nabla f(w_{t}, x_{\pi^{q_{t}}_{\pi_{r_{t}}}})\\
    &=\alpha \sum_{t'=1}^t(1-\alpha)^{t-t'}  \nabla f(w_{t'}, x_{\pi^{q_{t'}}_{\pi_{r_{t'}}}})
\end{align*}
Thus our momentum $m_t$ in Algorithm \ref{alg:normalized} is exactly as the first term in the output of Algorithm \ref{alg:gaussianaggregation} (Eq. \ref{eqn:renyinoise}) and $\noise_t$ is the same as the second term. Then, by Theorem \ref{thm:renyiaggregation}, Algorithm \ref{alg:normalized} is $\left(z, \tfrac{z}{2\sigma^2}\right)-$RDP for all $z$.
\end{proof}
\subsection{Utility}
To prove the main theorem of section \ref{sec:1/3bound} (Theorem \ref{thm:1/3bound}) we would need some extra lemmas on the momentum error below. Lemma \ref{lemma:populationerror} is the bound on the error of $m_t$ without any added noise from the tree. This error comes from the biasedness of the momentum as well as shuffling. Then, we will prove Lemma \ref{thm:momentumaccuracy} which is the bound on the added noise.
\begin{Lemma}\label{lemma:populationerror}
Define $m_t$ as:
\begin{align*}
    m_t &= (1-\alpha)m_{t-1}+\alpha \nabla f(w_t, x_{\pi^{q_t}_{r_t}})
\end{align*}
where $x^{q_t}_{\pi_{r_t}}$ is the sample at iteration $t = q_tN+r_t$. Let:
\begin{align*}
     \epsilon_t &= m_t - \nabla F(w_t)
\end{align*}
Then:
\begin{align*}
    \E[\| \epsilon_t\|] &\le2G\sqrt{\alpha} + 2\eta NL + \frac{\eta L}{\alpha}
\end{align*}
\end{Lemma}
\begin{proof}
Let $i = q_iN + r_i$ for any $ i \in [t]$. Then:
\begin{align}
    \E[\| \epsilon_t\|] &= \E[\|m_t - \nabla F(w_t)\|]\nonumber\\
    &\le \E[\|m_t - \alpha\sum_{i=1}^{t}(1-\alpha)^{t-i}\nabla F(w_{q_iN}) \|]+ \E[\|\alpha\sum_{i=1}^{t}(1-\alpha)^{t-i}(\nabla F(w_{q_iN}) - \nabla F(w_i))\|] \nonumber\\
    &\quad\quad+ \E[\|\alpha\sum_{i=1}^{t}(1-\alpha)^{t-i}\nabla F(w_{i}) - \nabla F(w_{t})\| ]\nonumber\\
    &\le  \E[\|m_t - \alpha\sum_{i=1}^{t}(1-\alpha)^{t-i}\nabla F(w_{q_iN}) \|] + \eta NL + \E[\|\alpha\sum_{i=1}^{t}(1-\alpha)^{t-i}\nabla F(w_{i}) - \nabla F(w_{t})\| ]\nonumber\\
    &= \E[\|\alpha\sum_{i=1}^{t}(1-\alpha)^{t-i}\nabla f(w_{i}, x^{q_i}_{\pi_{r_i}} ) - \alpha\sum_{i=1}^{t}(1-\alpha)^{t-i}\nabla F(w_{q_iN}) \|] + \eta NL\nonumber\\
    &\quad\quad+ \E[\|\alpha\sum_{i=1}^{t}(1-\alpha)^{t-i}\nabla F(w_{i}) - \nabla F(w_{t})\| ]\nonumber\\
    &\le \E[\| \alpha\sum_{i=1}^{t}(1-\alpha)^{t-i}(\nabla f(w_{q_iN}, x^{q_i}_{\pi_{r_i}} ) - \nabla F(w_{q_iN})) \|] + \E[\| \alpha\sum_{i=1}^{t}(1-\alpha)^{t-i}(\nabla f(w_{i}, x^{q_i}_{\pi_{r_i}} ) - \nabla f(w_{q_iN}, x^{q_i}_{\pi_{r_i}} )) \|] \nonumber\\
    &\quad\quad+ \eta NL +  \E[\|\alpha\sum_{i=1}^{t}(1-\alpha)^{t-i}\nabla F(w_{i}) - \nabla F(w_{t})\| ]\nonumber\\
      &\le \sqrt{\E[\|\alpha\sum_{i=1}^{t}(1-\alpha)^{t-i}(\nabla f(w_{q_iN}, x^{q_i}_{\pi_{r_i}} ) - \nabla F(w_{q_iN})) \|^2]} + 2\eta NL  \nonumber\\
      &\quad\quad+  \E[\|\alpha\sum_{i=1}^{t}(1-\alpha)^{t-i}\nabla F(w_{i}) - \nabla F(w_{t})\| ] \label{eqn:errornorm1}
\end{align}
First let us bound the last term in Eq.\ref{eqn:errornorm1}. Denote $g_t = (1-\alpha)g_{t-1} + \alpha \nabla F(w_t)$. Then, if we let $g_1 = \nabla F(w_1)$, $g_t = \alpha\sum_{i=1}^{t}(1-\alpha)^{t-i}\nabla F(w_{i}) $. Let $r_t= \E[\|\alpha\sum_{i=1}^{t}(1-\alpha)^{t-i}\nabla F(w_{i}) - \nabla F(w_{t})\| ] = \E[\|g_t - \nabla F(w_t)\|]$, we have:
\begin{align*}
    r_t &= \E\left[\|(1-\alpha)g_{t-1} + \alpha \nabla F(w_t) - \nabla F(w_t)\|\right]\\
    &= \E\left[\|(1-\alpha)(g_{t-1} - \nabla F(w_t))\|\right]\\
    &=\E\left[ \|(1-\alpha)(g_{t-1} - \nabla F(w_{t-1})) + (1-\alpha)(\nabla F(w_{t-1}) - \nabla F(w_t))\|\right]
    \intertext{Unroll the recursive expression:}
    &= \E\left[\|\sum_{i=1}^t(1-\alpha)^{t-i}((\nabla F(w_{i-1}) - \nabla F(w_i)))\|\right]\\
    &\le \frac{\eta L}{\alpha}
\end{align*}
Now let us bound $\E[\|\alpha\sum_{i=1}^{t}(1-\alpha)^{t-i}(\nabla f(w_{q_iN}, x^{q_i}_{\pi_{r_i}} ) - \nabla F(w_{q_iN})) \|^2]$. For any iteration $i$ let $A_{\pi_{r_i}^{q_i}} = \nabla f(w_{q_iN}, x^{q_i}_{\pi_{r_i}} ) - \nabla F(w_{q_iN})$ and $c_i = (1-\alpha)^{t-i}$. Thus:
\begin{align*}
    \E[\|\alpha\sum_{i=1}^{t}(1-\alpha)^{t-i}(\nabla f(w_{q_iN}, x^{q_i}_{\pi_{r_i}} ) - \nabla F(w_{q_iN})) \|^2] &= \alpha^2\E\left[\|\sum_{i=1}^{t}c_iA_{\pi_{r_i}^{q_i}}\|^2\right]
\end{align*}
If we expand the equation above, we will have some cross terms as well as some squared norm terms. First, let us examine the cross terms for iteration $i < j$ where $q_i = q_j = q$. Then:
\begin{align*}
    \E\left[c_ic_j\langle A_{\pi_{r_i}^{q}}, A_{\pi_{r_j}^{q}} \rangle\right] &=\sum_{k_i =1}^N c_ic_j\E_{\pi_{r_j}^{q}}\left[\langle A_{\pi_{r_i}^{q}}, A_{\pi_{r_j}^{q}} \rangle |\pi_{r_i}^{q} = k_i\right]P[\pi_{r_i}^{q} = k_i]\\
    &=\sum_{k_i =1}^Nc_ic_jP[\pi_{r_i}^{q} = k_i]\left\langle A_{k_i},\E_{\pi_{r_j}^{q}}\left[ A_{\pi_{r_j}^{q}}  |\pi_{r_i}^{q} = k_i\right]\right\rangle \\
     &= \sum_{k_i =1}^Nc_ic_jP[\pi_{r_i}^{q} = k_i]\left\langle A_{k_i},\frac{\sum_{k_j \ne  k_i}A_{k_j} }{N-1}\right\rangle 
\end{align*}
Notice that $\sum_{k_j}A_{k_j} = \sum_{k_j=1}^N \nabla f(w_{qN}, x^q_{\pi_{k_j}}) - \nabla F(w_{qN}) = 0$ (since the iterate at the beginning of the epoch is independent of the data that are samples \textit{in} that particular epoch). Thus $\sum_{k_j \ne k_i} A_{k_j} = - A_{k_i} $. Then:
\begin{align*}
    \sum_{k_i=1}^Nc_ic_jP[\pi^q_{r_i} = k_i]\left\langle A_{k_i},\frac{\sum_{k_j \ne  k_i}A_{k_j} }{N-1}\right\rangle &= \sum_{k_i=1}^Nc_ic_jP[\pi^q_{r_i} = k_i]\left\langle A_{k_i},\frac{-A_{k_i}}{N-1}\right\rangle\\
    &= \sum_{k_i=1}^NP[\pi^q_{r_i} = k_i]\frac{-c_ic_j}{N-1}\|A_{k_i}\|^2 \\
    &\le 0
\end{align*}
Now let us analyze the cross terms for $i < j$ where $q_i < q_j$.
\begin{align*}
    \E\left[c_ic_j\langle A_{\pi^{q_i}_{r_{i}}},A_{\pi^{q_j}_{r_{j}}} \rangle\right] &=\sum_{k_i=1}^Nc_ic_j\E_{\pi^{q_j}}\left[\langle A_{\pi^{q_i}_{r_{i}}},A_{\pi^{q_j}_{r_{j}}} \rangle |\pi^{q_i}_{r_{i}} = k_i \right]P[\pi^q_{r_i} = k_i] \\
    &= \sum_{k_i=1}^Nc_ic_jP[\pi^q_{r_i} = k_i]\left\langle A_{\pi^{q_i}_{r_{i}}},\E_{\pi^{q_j}}\left[A_{\pi^{q_j}_{r_{j}}}  |\pi^{q_i}_{r_{i}} = k_i \right]\right\rangle  \\
    &=\sum_{k_i=1}^Nc_ic_jP[\pi^q_{r_i} = k_i]\left\langle A_{\pi^{q_i}_{r_{i}}},\frac{\sum_{k_j =1 }^NA_{k_j}}{N}\right\rangle \\
    &= 0
\end{align*}
Thus, the cross terms $\E\left[c_ic_j\langle A_{\pi^{q_i}_{r_{i}}},A_{\pi^{q_j}_{r_{j}}} \rangle\right] \le 0$ for every $i < j$. Either way,
\begin{align*}
    \alpha^2\E\left[\|\sum_{i=1}^{t}c_iA_{\pi_{r_i}^{q_i}}\|^2\right] &\le \alpha^2\E\left[\sum_{i=1}^{t}c_i^2\|A_{\pi_{r_i}^{q_i}}\|^2\right]\\
    &\le 4G^2\alpha^2 \sum_{i=1}^{t}(1-\alpha)^{2(t-i)}\\
    &\le 4\alpha G^2
\end{align*}
Plugging this back to eq. \ref{eqn:errornorm1}
\begin{align*}
    \E[\| \epsilon_t\|] &\le 2G\sqrt{\alpha} + 2\eta NL + \frac{\eta L}{\alpha}
\end{align*}
\end{proof}
\begin{restatable}{Lemma}{momentumaccuracy}\label{thm:momentumaccuracy}
Let $V=(\min(R, \lfloor\log_2(N)\rfloor)+1)\frac{T}{N}+ \sum_{j=\lfloor\log_2(N)+1\rfloor}^{R}\lfloor \frac{T}{2^j}\rfloor$ as in Algorithm~\ref{alg:normalized}. Suppose $\alpha\ge \frac{1}{N}$. Then, 
\begin{align*}
    \E\left[\|\hat m_t - m_t\|\right] &\le 4\alpha G\sigma\sqrt{dV\log_2T}
\end{align*}
\end{restatable}
\begin{proof}
We have:
\begin{align*}
    \hat m_t &= m_t + \noise_t
\end{align*}
where $\noise_t = \sum_{[y,z]\in \COMPOSE(1,t)}\zeta_{[y,z]}$. Since there are at most $\log_2 T$ intervals in $\COMPOSE(1,t)$, $\noise_t$ is a Gaussian random vector with variance:
\begin{align}
    \var &\le 16\alpha^2G^2\sigma^2V\log_2T
\end{align}
Then:
\begin{align*}
    \E\left[\|\hat m_t - m_t\|\right] &\le \sqrt{d}\sqrt{\var}\\
    &= 4\alpha G\sigma\sqrt{dV\log_2T}
\end{align*}
\end{proof}
\normalizedbound*
\begin{proof}
From Lemma \ref{lemma:gradientbound}, we have:
\begin{align*}
    \E\left[\|\nabla F(\hat w)\|\right]&\le \frac{3\E\left[(F(w_1) - F(w_{T+1}))\right]}{2\eta T} + \frac{3L\eta}{4} + \frac{3}{T}\sum_{t=1}^T\E[\|\hat \epsilon_t\|] \\
    &= \frac{3\E\left[(F(w_1) - F(w_{T+1}))\right]}{2\eta T} + \frac{3L\eta}{4} + \frac{3}{T}\sum_{t=1}^T\E[\|\hat m_t - \nabla F(w_t)\|] \\
    &= \frac{3\E\left[(F(w_1) - F(w_{T+1}))\right]}{2\eta T} + \frac{3L\eta}{4} + \frac{3}{T}\sum_{t=1}^T\E[\|\hat m_t - m_t + m_t  -\nabla F(w_t)\|] \\
    &\le \frac{3\E\left[(F(w_1) - F(w_{T+1}))\right]}{2\eta T} + \frac{3L\eta}{4} + \frac{3}{T}\sum_{t=1}^T \E\left[\|\hat m_t - m_t\|\right] + \E[\|m_t -\nabla F(w_t)\|]\\
     &\le \frac{3\E\left[(F(w_1) - F(w_{T+1}))\right]}{2\eta T} + \frac{3L\eta}{4} + \frac{3}{T}\sum_{t=1}^T \E\left[\|\hat m_t - m_t\|\right] + \E[\| \epsilon_t\|]
\end{align*}
Applying Lemma \ref{lemma:populationerror} and
using Theorem \ref{thm:momentumaccuracy} with $V =(\min(R, \lfloor\log_2(N)\rfloor)+1)\frac{T}{N}+ \sum_{j=\lfloor\log_2(N)+1\rfloor}^{R}\lfloor \frac{T}{2^j}\rfloor \le 4\log_2 T\frac{T}{N}$, $\sigma = \frac{2\sqrt{\log(1/\delta)}}{\epsilon}$:
\begin{align*}
     \E\left[\|\nabla F(\hat w)\|\right]&\le \frac{3R}{2\eta T} + \frac{3\eta L}{4} +  6G\sqrt{\alpha } + 6\eta NL + \frac{3\eta L}{\alpha} + \frac{12\alpha G\log_2 T\sqrt{dT \log(1/\delta)}}{\epsilon\sqrt{N}}\\
     &\le \frac{3R}{2\eta T}  +  6G\sqrt{\alpha } + 6\eta NL + \frac{6\eta L}{\alpha} + \frac{12\alpha G\log_2 T\sqrt{dT \log(1/\delta)}}{\epsilon\sqrt{N}}
\end{align*}
Let $\eta = \frac{1}{\sqrt{NT}}$:
\begin{align*}
     \E\left[\|\nabla F(\hat w)\|\right]&\le \frac{(\frac{3}{2}R+6L)\sqrt{N}}{\sqrt{T}} + \frac{6L}{\alpha\sqrt{NT}} +6G\sqrt{\alpha} +\frac{6G}{\alpha T} + \frac{12\alpha G\log_2 T\sqrt{dT \log(1/\delta)}}{\epsilon\sqrt{N}}
\end{align*}
Set $\alpha = \frac{\epsilon N}{T\log_2 T\sqrt{d\log(1/\delta)}}$:
\begin{align*}
    \E\left[\|\nabla F(\hat w)\|\right] &\le \frac{(\frac{3}{2}R+6L + 12G)\sqrt{N}}{\sqrt{T}} + \frac{6L\sqrt{T}\log_2T\sqrt{d\log(1/\delta)}}{\epsilon N\sqrt{N}}+ \frac{6G\sqrt{\epsilon N}}{(d\log (1/\delta))^{1/4}\sqrt{T\log_2 T}}\\
    &\quad\quad+ \frac{6G\log_2 T\sqrt{d\ \log(1/\delta)}}{\epsilon N} 
\end{align*}
Since $\alpha \ge \frac{1}{N}$, then the largest $T \le \frac{\epsilon N^2}{\log_2 T\sqrt{d\log(1/\delta)}}$:
\begin{align*}
    \E\left[\|\nabla F(\hat w)\|\right]&\le \frac{(\frac{3}{2}R+6L\sqrt{\log_2T\log(1/\delta)} + 12G\sqrt{\log_2T}\log(1/\delta)^{1/4})d^{1/4}}{\sqrt{\epsilon N}} + \frac{3L(d\log(1/\delta))^{1/4}\sqrt{\log_2 T}}{4\sqrt{\epsilon}N^{3/2}} + \frac{6G}{\sqrt{N}} \\
    &\quad\quad+ \frac{6G\log_2 T\sqrt{d\log(1/\delta)}}{\epsilon N} \\
    &\le \tilde O\left(\frac{d^{1/4}}{\sqrt{\epsilon N}} + \frac{1}{\sqrt{N}}\right)
\end{align*}
\end{proof}
\section{Proof of section \ref{sec:cross-epoch}}
\subsection{Privacy}
\crossprivacy*
\begin{proof}
To show the privacy guarantee, we will show that the releasing all of the $\hat F^G_{[y,z]}$,  $\hat F^\Delta_{[y,z]}$, $\hat F^r_{[y,z]}$   is private. To see this, observe that the intervals $[y,z]$ correspond to nodes of a binary tree with at least $T$ leaves: $[y,z]$ is the node whose descendents are the  leaves $y,\dots,z$, so that  we are essentially analyzing a standard tree-based aggregation mechanism.

To start, let us re-define the queries:
\begin{align*}
     F^G_{[a,b]}(x_1,\dots,x_N) &= \sum_{t=a}^b (1-\alpha)^{b-t} \nabla f(w_{q_tN}, x_{\pi^{q_t}_{r_t}}) \\
    F^\Delta_{[a,b]}(x_1,\dots, x_N) &= \sum_{t=a}^b (1-\alpha)^{b-t}(\nabla f(w_{q_tN}, x_{\pi^{q_t}_{r_t}})  - \nabla f(w_{(q_{t}-1)N}, x_{\pi^{q_t}_{r_t}}))\\
    F^r_{[a,b]}(x_1,\dots, x_N) &= \sum_{t=a}^b (1-\alpha)^{b-t}(\nabla f(w_{t}, x_{\pi^{q_t}_{r_t}})  - \nabla f(w_{q_{t}N}, x_{\pi^{q_t}_{r_t}}))
\end{align*}

First, we will compute the sensitivity of $F^G_{[a,b]}$ by the exact same analysis as in Theorem~\ref{thm:momentumprivacy}, we have that $F^G_{[a,b]}$ has sensitivity
\begin{align*}
    2 G\sum_{j=0}^{\lceil |b-a|/N\rceil -1}(1-\alpha)^{jN}\le 4 G
\end{align*}
where the last inequality uses Proposition \ref{prop:powersum} in conjunction with the assumption $\alpha\ge \frac{1}{N}$.

Then, for the sensitivity of $F^\Delta_{[a,b]}$, observe that
\begin{align*}
    \|\nabla f(w_{q_tN}, x_{\pi^{q_t}_{r_t}})  - \nabla f(w_{q_{t'}N}, x_{\pi^{q_t}_{r_t}})\|\le \eta NL
\end{align*}
Thus, by essentially the same argument used to  bound the sensitivity of $F^G_{[a,b]}$, $F^{\Delta}_{[a,b]}$  has sensitivity at most $2 \eta NL$.

Finally, for the sensitivity of $F^r_{[a,b]}$:
\begin{align*}
    \|\nabla f(w_{t}, x_{\pi^{q_t}_{r_t}})  - \nabla f(w_{q_{t}N}, x_{\pi^{q_t}_{r_t}})\| &\le \eta NL
\end{align*}
Thus, $F^r_{[a,b]}$ also has sensitivity $2\eta NL$.

Next, observe that for any index  $i\in [1,N]$, $x_i$ can influence at  most $1+\lfloor \log_2(2N)\rfloor \le  3\log_2(N)$ intervals $[a,b]$ corresponding to nodes in the tree such that $b\le 2N$. Thus, by adding Gaussian noise with standard deviation $\sqrt{V_{\le 2N}}\sigma \delta_G$ to $F^G_{[a,b]}$ and $\sqrt{V_{\le 2N}} \sigma \delta_\Delta$ to $F^\Delta_{[a,b]}$ for any $[a,b]\subset [1,2N]$ yields a set  of estimates that are $(z,z/2\sigma^2)$ Renyi-differentially private for all $z$.

Now we turn to intervals that are not subsets of $[1,2N]$. For these, notice that again by same  analysis used to  prove Theorem~\ref{thm:momentumprivacy}, we have that the  number  of nodes any index $i$ can influence is at most:
\begin{align*}
    (\lfloor \log_2(N)\rfloor  + 1)\frac{T}{N} + \sum_{j=\lfloor \log_2(N)+1\rfloor}^{\lceil\log_2(T)\rceil}\lfloor \frac{T}{2^j}\rfloor \le 4\frac{T}{N}\log_2(N)
\end{align*}
Thus, by adding Gaussian noise with standard deviation $\sqrt{V_{>2N}}\sigma \delta_G$ or $\sqrt{V_{>2N}}\sigma \delta_\Delta$ to $F^G_{[a,b]}$, or $F^\Delta_{[a,b]}$, we obtain $(z,z/2\sigma^2)$ Renyi-differential privacy.
Finally, by adding noise with standard deviation $\sqrt{V_{>2N}}\sigma \delta_r$ to $F^r_{[a,b]}$, we also obtain $(z,z/2\sigma^2)-$RDP. Therefore, overall the mechanism is $(z,3z/2\sigma^2)-$RDP.

For the $(\epsilon,\delta)$-DP guarantee, notice that $(z,3z/2\sigma^2)$-RDP implies $(\epsilon,\delta)$-DP with  $\epsilon = 3/2\sigma^2 + 2\sqrt{3\log(1/\delta)/2\sigma^2}$ for all $\delta$. Thus, any when $\delta \ge \exp(-\epsilon)$, any $\sigma$ satisfying 

\begin{align*}
    \sigma\ge \frac{4\sqrt{\log(1/\delta)}}{\epsilon}
\end{align*}
will suffice to achieve the desired privacy.

\end{proof}
\subsection{Utility}
First, let us prove a bound on $\E[\|\hat m_t - m_t\|]$ by bounding the variance of all of the added noises.
\crossepocherorr*
\begin{proof}
First, observe that  since $\alpha\ge 1/N$, by Proposition~\ref{prop:powersum}, we have:
\begin{align}
    \sum_{i=0}^\infty (1-\alpha)^{iN}\le \frac{1}{1-\exp(-1)}\le 2\label{eqn:alphabound}
\end{align}

Now, we define some notation. Let For any iteration t, let $t= q_tN+r_t$. For any interval $[a,b]$,  set
\begin{align*}
    F^G_{[a,b]}(x_1,\dots,x_N) &= \sum_{t=a}^b (1-\alpha)^{b-t} \nabla f(w_{q_tN}, x_{\pi^{q_t}_{r_t}}) \\
    F^\Delta_{[a,b]}(x_1,\dots, x_N) &= \sum_{t=a}^b (1-\alpha)^{b-t}(\nabla f(w_{q_tN}, x_{\pi^{q_t}_{r_t}})  - \nabla f(w_{(q_{t}-1)N}, x_{\pi^{q_t}_{r_t}}))\\
    F^r_{[a,b]}(x_1,\dots, x_N) &= \sum_{t=a}^b (1-\alpha)^{b-t}(\nabla f(w_{t}, x_{\pi^{q_t}_{r_t}})  - \nabla f(w_{q_{t}N}, x_{\pi^{q_t}_{r_t}}))
\end{align*}
We will use the notation $F^G_{[a,b]}$ to indicate $F^G_{[a,b]}(x_1,\dots,x_N)$ for  brevity.
Then:
\begin{align*}
    G_{[a,b]}&=\sum_{[y,z]\in \COMPOSE(a,b)} (1-\alpha)^{b-z} F^G_{[y,z]}\\
    \Delta_{[a,b]}&=\sum_{[y,z]\in \COMPOSE(a,b)} (1-\alpha)^{b-z} F^\Delta_{[y,z]}\\
    r_{[a,b]}&=\sum_{[y,z]\in \COMPOSE(a,b)} (1-\alpha)^{b-z} F^r_{[y,z]}
\end{align*}

Observe that with this definition, if $m_t$ is the  momentum value defined recursively as $m_{t+1}= (1-\alpha)m_t + \alpha \nabla f(w_t, x_{\pi^{q_t}_{r_t}})$, we have:
\begin{align*}
   m_t&= \alpha\sum_{i=1}^{t}(1-\alpha)^{t-i}\nabla f(w_i,  x_{\pi^{q_i}_{r_i}})\\
   &=  \alpha\sum_{i=1}^{t}(1-\alpha)^{t-i}\nabla f(w_{q_iN},  x_{\pi^{q_i}_{r_i}}) + \alpha\sum_{i=1}^{t}(1-\alpha)^{t-i}(\nabla f(w_i,  x_{\pi^{q_i}_{r_i}}) - \nabla f(w_{q_iN},  x_{\pi^{q_i}_{r_i}}) )\\
   &= \alpha G_{[1,t]} + \alpha r_{[1,t]}
\end{align*}

From Lemma \ref{lem:breaksum}, we know that:
\begin{align}
    \alpha G_{[1,t]} &= (1-\alpha)^{t-r}\alpha G_{[1,r]}+\sum_{i=0}^{q-1}(1-\alpha)^{Ni}(1-\gamma)^{q-(i+1)}\alpha G_{[r+1,r+N]} \nonumber\\
    &\quad\quad+\sum_{i=1}^{q-1}\left(\sum_{j=0}^{i-1} (1-\gamma)^j(1-\alpha)^{(i-1-j)N}\right)\alpha \left((1-\gamma)\Delta_{[t-iN+1,t-(i-1)N]} + \gamma G_{[t-iN+1,t-(i-1)N]}\right)\label{eqn:breaksum}
\end{align}
Now we can compute the accuracy of the estimate $\alpha \hat G_{[1,t]}$ using $\delta_\Delta$, $\delta_G$, $\delta_r$, $V_{\le 2N} = 3\log_2 N$, $V_{>2N} = 4\log_2 N \frac{T}{N}$. Let us analyze the noise added to $\alpha \hat G_{[1,t]}$ term by term.

\textbf{First term, $(1-\alpha)^{T-r}\alpha G_{[1,r]}$:}
\begin{align*}
     (1-\alpha)^{T-r}\alpha\hat G_{[1,r]} &= (1-\alpha)^{T-r}\alpha\sum_{[y,z]\in \COMPOSE(1,r)} (1-\alpha)^{r-z} F^G_{[y,z]}+ (1-\alpha)^{T-r}\alpha\sum_{[y,z]\in \COMPOSE(1,r)} (1-\alpha)^{r-z}\zeta^G_{[y,z]}\\
     &=(1-\alpha)^{T-1}\alpha G_{[1:r]}+(1-\alpha)^{T-r}\alpha\sum_{[y,z]\in \COMPOSE(1,r)} (1-\alpha)^{r-z}\zeta^G_{[y,z]}
\end{align*}
where $\zeta^G_{[y,z]} \sim N(0, \delta_G^2\sigma^2V)$ with $V_{\le 2N} = 3\log_2 N$. Now, by Proposition~\ref{prop:composebound}, there are at most $2(1+\log_2 r)\le 4\log_2(N)$ intervals in $\COMPOSE(1,r)$. Thus, the variance of the noise added to the first term is:
\begin{align*}
    \var_1 &\le 2(1-\alpha)^{T-r}(1+\log_2 (r)) \alpha^2\delta_G^2\sigma^2V\\
    &\le 12(1-\alpha)^{T-r} \alpha^2\delta_G^2\sigma^2 \log_2^2N\\
    &\le 12\alpha^2\delta_G^2\sigma^2 \log_2^2N
\end{align*}

\textbf{Second term, $\sum_{i=0}^{q-1}(1-\alpha)^{Ni}(1-\gamma)^{q-(i+1)}\alpha G_{[r+1,r+N]}$:}
\begin{align*}
    \sum_{i=0}^{q-1}(1-\alpha)^{Ni}(1-\gamma)^{q-(i+1)}\alpha G_{[r+1,r+N]} &= \sum_{i=0}^{q-1}(1-\alpha)^{Ni}(1-\gamma)^{q-(i+1)}\alpha  \sum_{[y,z]\in \COMPOSE(r+1,r+N)}\left(F^G_{[y,z]} +\zeta^G_{[y,z]}\right)
\end{align*}
Since there are at most $4\log_2 N$ terms in $\COMPOSE(1,N)$, the variance of the noise added to the second term is:
\begin{align*}
    \var_2 &\le  4\sum_{i=0}^{q-1}(1-\alpha)^{2Ni}(1-\gamma)^{2q-2(i+1)}\log_2 N\alpha^2\delta_G^2\sigma^2 V_{\le 2N}\\
    &\le \frac{4\log_2 N\alpha^2\delta_G^2\sigma^2 V_{\le 2N}}{(1-\exp(-1))^2}\\
    &= \frac{12\alpha^2\delta_G^2\sigma^2\log^2_2N}{(1-\exp(-1))^2}
\end{align*}
where the second inequality comes from Proposition \ref{prop:advpowersum}.

\textbf{Third term, $\alpha\sum_{i=1}^{q-1}\left(\sum_{j=0}^{i-1} (1-\gamma)^j(1-\alpha)^{(i-1-j)N}\right) \gamma G_{[T-iN+1,T-(i-1)N]}$:}
\begin{align*}
     &\sum_{i=1}^{q-1}\left(\sum_{j=0}^{i-1} (1-\gamma)^j(1-\alpha)^{(i-1-j)N}\right) \alpha\gamma \hat G_{[T-iN+1,T-(i-1)N]} \\
     \qquad&= \sum_{i=1}^{q-1}\left(\sum_{j=0}^{i-1} (1-\gamma)^j(1-\alpha)^{(i-1-j)N}\right) \alpha\gamma \sum_{[y,z] \in \COMPOSE(T-iN+1, T-(i-1)N)}(1-\alpha)^{T-(i-1)N-z}(F^G_{[y,z]}+\zeta^{G}_{[y,z]}) 
\end{align*}
where $\zeta^{G}_{[y,z]} \sim N(0, \delta_G^2\sigma^2V_{>2N})$ where $V_{>2N} = 4\frac{T}{N}\log_2N$. There are at still at most $4\log_2(N)$ terms in $\COMPOSE(T-iN+1, T-(i-1)N)$, so Using Corollary \ref{cor:sumgamma}, the variance of the noise added to the third term is:
\begin{align*}
    \var_3 &\le 4\sum_{i=1}^{q-1}\frac{(1-\gamma)^{2(i-1)}}{(1-\exp(-1))^2}\alpha^2\gamma^2\log_2 N \delta_G^2\sigma^2V_{>2N}\\
    &\le \frac{16\gamma\alpha^2\log_2^2 N \delta_G^2\sigma^2T }{(1-\exp(-1))^2N}\\
    &\le \frac{16\gamma\alpha^2 \delta_G^2\sigma^2T \log_2^2 N }{(1-\exp(-1))^2N}
\end{align*}

\textbf{Fourth Term, $\sum_{i=1}^{q-1}\left(\sum_{j=0}^{i-1} (1-\gamma)^j(1-\alpha)^{(i-1-j)N}\right)\alpha(1-\gamma)\Delta_{[T-iN+1,T-(i-1)N]}$:}

Similar to the third term, we have:
\begin{align*}
    &\sum_{i=1}^{q-1}\left(\sum_{j=0}^{i-1} (1-\gamma)^j(1-\alpha)^{(i-1-j)N}\right) \alpha(1-\gamma)\hat \Delta_{[T-iN+1,T-(i-1)N]} \\
    \quad\quad&=\sum_{i=1}^{q-1}\left(\sum_{j=0}^{i-1} (1-\gamma)^j(1-\alpha)^{(i-1-j)N}\right) \alpha(1-\gamma)\sum_{[y,z] \in \COMPOSE(T-iN+1, T-(i-1)N)}(1-\alpha)^{T-(i-1)N-z}(\Delta_{[y,z]}+\zeta^{\Delta}_{[y,z]}) 
\end{align*}
Thus the variance of the noise added to the fourth term is:
\begin{align*}
    \var_4 &\le4 \sum_{i=1}^{q-1}\alpha^2\frac{(1-\gamma)^{2i}}{(1-\exp(-1))^2}\log_2 N \delta_\Delta^2\sigma^2V_{>2N}\\
    &\le \frac{16\alpha^2\log_2^2 N \delta_\Delta^2\sigma^2 T }{\gamma N(1-\exp(-1))^2}\\
    &\le \frac{16\alpha^2 \delta_\Delta^2\sigma^2 T \log^2_2 N}{\gamma N(1-\exp(-1))^2}
\end{align*}
Now let us analyze $\alpha r_{[1,t]}$ to see how much noise we need to add to make it private. We have:
\begin{align*}
    \alpha \hat r_{[1,t]} &= \alpha r_{[1,t]} +\alpha \sum_{[y,z] \in \COMPOSE(1,t)} \zeta^r_{[y,z]}
\end{align*}
where $\zeta^r_{[y,z]} \sim N(0, \delta^2_r\sigma^2V_{>2N})$. There are at most $3\log_2T$ in $\COMPOSE(1,t)$, thus the variance of the noise added is:
\begin{align*}
    \var_5 &= 3\log_2T\alpha^2\delta^2_r\sigma^2V_{>2N}
    \intertext{Plug in $V_{>2N} = 4\frac{T}{N}\log_2N$, $\delta^r = 2\eta NL$}
    &\le \frac{48T\alpha^2\eta^2N^2L^2\sigma^2\log_2^2T}{ N}
\end{align*}
Now combining $\var_1, \var_2, \var_3, \var_4$,$\var_5$ we have the total variance of the Gaussian noise added to $\hat m_t$ is:
\begin{align*}
    \var&\le 12\alpha^2\delta_G^2\sigma^2 \log^2_2N + \frac{12\alpha^2\delta_G^2\sigma^2\log^2_2N}{(1-\exp(-1))^2} + \frac{16\gamma\alpha^2 \delta_G^2\sigma^2T \log^2_2 N }{(1-\exp(-1))^2N} +  \frac{16\alpha^2 \delta_\Delta^2\sigma^2 T \log^2_2 N}{(1-\exp(-1))^2\gamma N} \\
    &\quad\quad+  \frac{48T\alpha^2\eta^2N^2L^2\sigma^2\log_2^2T}{ N}\\
    &= 256G^2\alpha^2\sigma^2 \log^2_2N+ \frac{256G^2\alpha^2\sigma^2\log^2_2N}{(1-\exp(-1))^2} + \frac{256G^2\gamma\alpha^2 \sigma^2T \log^2_2N }{(1-\exp(-1))^2N} + \frac{1024\eta^2N^2L^2\alpha^2 \sigma^2 T \log^2_2 N}{(1-\exp(-1))^2\gamma N} \\
    &\quad\quad+  \frac{48T\alpha^2\eta^2N^2L^2\sigma^2\log_2^2T}{ N}
\end{align*}
Since the added noise is a Gaussian vector:
\begin{align*}
    &\E[\|\hat m_t - m_t\|] \\
    &\le \sqrt{d}\sqrt{\var}\\
    &\le \sqrt{d} \times \left(16G\alpha \sigma \log_2 N + \frac{16G\alpha \sigma \log_2 N}{1-\exp(-1)} + \frac{16G\alpha \sqrt{\gamma}\sigma \sqrt{T}\log_2 N}{(1-\exp(-1))\sqrt{N}} + \frac{32\eta NL\alpha \sigma \sqrt{T}\log_2N}{(1-\exp(-1))\sqrt{\gamma N}} + \frac{4\eta NL\alpha \sigma\log_2T\sqrt{3T}}{\sqrt{N}}\right)\\
    &\le \sqrt{d} \times \left(48G\alpha \sigma \log_2 T +  \frac{32G\alpha \sqrt{\gamma}\sigma \sqrt{T}\log_2 T}{\sqrt{N}} + \frac{64\eta NL\alpha \sigma \sqrt{T}\log_2 T}{\sqrt{\gamma N}}+  \frac{4\eta NL\alpha \sigma\log_2T\sqrt{3T}}{\sqrt{N}}\right)
\end{align*}
Set $\gamma = \eta N$:
\begin{align*}
 \E[\|\hat m_t- m_t\|]
     &\le \sqrt{d} \times \left(48G\alpha \sigma \log_2 T +  32G\alpha\sigma \sqrt{\eta T}\log_2 T+ 64 L\alpha \sigma \sqrt{\eta T}\log_2 T +\frac{4\eta NL\alpha \sigma\log_2T\sqrt{3T}}{\sqrt{N}} \right)
\end{align*}
Setting $\sigma = \frac{4\sqrt{\log (1/\delta)}}{\epsilon}$:
\begin{align*}
    \E[\|\hat m_t- m_t\|]
     &\le \frac{192G\alpha \log_2 T\sqrt{d\log(1/\delta)} }{\epsilon} + \frac{128(G+2L)\alpha \sqrt{\eta dT\log(1/\delta)}\log_2 T}{\epsilon} + \frac{16\eta L\alpha \log_2T\sqrt{3dTN\log(1/\delta)}}{\epsilon} 
\end{align*}
\end{proof}
\crossutility*
\begin{proof}
From Lemma \ref{lemma:momentumaccuracy2/5bound}, we have:
\begin{align*}
    \E[\|\hat m_t - m_t\|]&\le  \frac{192G\alpha \log_2 T\sqrt{d\log(1/\delta)} }{\epsilon} + \frac{128(G+2L)\alpha \sqrt{\eta dT\log(1/\delta)}\log_2 T}{\epsilon}\\
    &\quad\quad+\frac{16\eta L\alpha \log_2T\sqrt{3dTN\log(1/\delta)}}{\epsilon} 
\end{align*}
Now use Lemma \ref{lemma:populationerror}, \ref{lemma:gradientbound} and let $K = 16\log_2 T\log(1/\delta)$:
\begin{align*}
     \E[\|\nabla F(\hat w)\|] &\le \frac{3R}{2\eta T} + \frac{3\eta L}{4} \\
     &\quad+\frac{3}{T}\sum_{t=1}^T\left( 2G\sqrt{\alpha} + 2\eta NL +\frac{\eta L}{\alpha} +\frac{12KG\alpha \sqrt{d} }{\epsilon} + \frac{8K(G+2L)\alpha \sqrt{\eta dT}}{\epsilon} + \frac{K\alpha\eta L\sqrt{3dTN}}{\epsilon}\right)\\
     &\le \frac{3R}{2\eta T} + \frac{3\eta L}{4} + 6G\sqrt{\alpha} + 6\eta NL + \frac{3\eta L}{\alpha} + \frac{36GK\alpha \sqrt{d} }{\epsilon} + \frac{24K(G+2L)\alpha \sqrt{\eta dT}}{\epsilon} +  \frac{3K\alpha\eta L\sqrt{3dTN}}{\epsilon}\\
     &\le \frac{3R}{2\eta T}  + 6G\sqrt{\alpha} + 6\eta NL + \frac{6\eta L}{\alpha} + \frac{36GK\alpha \sqrt{d} }{\epsilon} + \frac{24K(G+2L)\alpha \sqrt{\eta dT}}{\epsilon} +  \frac{3K\alpha\eta L\sqrt{3dTN}}{\epsilon}
\end{align*}
Set $\eta = \frac{1}{\sqrt{NT}}$:
\begin{align*}
    \E[\|\nabla F(\hat w)\|] &\le \frac{(\frac{3}{2}R+ 6L)\sqrt{N}}{\sqrt{T}}+ \frac{6L}{\alpha\sqrt{NT}}  + 6G\sqrt{\alpha} + \frac{36GK\alpha\sqrt{d}}{\epsilon} + \frac{24K(G+2L)\alpha T^{1/4}\sqrt{d}}{\epsilon N^{1/4}}+ \frac{3K\alpha L\sqrt{3d}}{\epsilon}\\
     &\le \frac{(\frac{3}{2}R+ 6L)\sqrt{N}}{\sqrt{T}}+ \frac{6L}{\alpha\sqrt{NT}}  + 6G\sqrt{\alpha} + \frac{36GK\alpha\sqrt{d}}{\epsilon} + \frac{24K(G+2L)\alpha T^{1/4}\sqrt{d}}{\epsilon N^{1/4}}+ \frac{9K\alpha L\sqrt{d}}{\epsilon}
\end{align*}
Set $\alpha = \frac{N^{3/4}\epsilon}{T^{3/4}\sqrt{d}}$:
\begin{align*}
    \E[\|\nabla F(\hat w)\|] &\le\frac{(\frac{3}{2}R+24K(G+2L) + 6L)\sqrt{N}}{\sqrt{T}}+ \frac{6LT^{1/4}\sqrt{d}}{\epsilon N^{5/4}}+ \frac{6GN^{3/8}\sqrt{\epsilon}}{T^{3/8}d^{1/4}} + \frac{36GKN^{3/4}}{T^{3/4}} + \frac{9KN^{3/4}L}{T^{3/4}}
\end{align*}
Because we must have $\alpha \ge \frac{1}{N}$, the largest value of $T = \frac{N^{7/3}\epsilon^{4/3}}{d^{2/3}}$:
\begin{align*}
    \E[\|\nabla F(\hat w)\|] &\le \frac{(\frac{3}{2}R+24K(G+2L)+ 12L)d^{1/3}}{(\epsilon N)^{2/3}}+ \frac{6G}{\sqrt{N}} + \frac{36GK\sqrt{d}}{\epsilon N} + \frac{9KL\sqrt{d}}{\epsilon N}\\
    &\le \tilde O\left(\frac{d^{1/3}}{(\epsilon N)^{2/3}} + \frac{1}{\sqrt{N}}\right)
\end{align*}
\end{proof}
\section{Technical Lemmas}
\gradientbound*
\begin{proof}
From smoothness:
\begin{align*}
    F(w_{t+1}) &\le F(w_t) + \langle\nabla F(w_t), w_{t+1} - w_t\rangle + \frac{L}{2}\|w_{t+1} -w_t\|^2 \\
    &=  F(w_t) + \eta\left\langle\nabla F(w_t), \frac{\hat m_t}{\|\hat m_t\|} \right\rangle + \frac{L\eta^2}{2} 
\end{align*}
Let's analyze the inner-product term via some case-work: Suppose $\|\hat \epsilon_t\| \le \frac{1}{2}\|\nabla F(w_t)\|$. Then $\frac{1}{2}\|\nabla F(w_t)\| \le \|\nabla F(w_t) +\hat \epsilon_t\| \le  \frac{3}{2}\|\nabla F(w_t)\| $ so that:
\begin{align*}
    -\left\langle\nabla F(w_t), \frac{\hat m_t}{\|\hat m_t\|} \right\rangle &= -\left\langle\nabla F(w_t), \frac{\nabla F(w_t) +\hat \epsilon_t}{\|\nabla F(w_t) +\hat \epsilon_t\|} \right\rangle \\
    &\le \frac{-\|\nabla F(w_t)\|^2}{\|\nabla F(w_t) +\hat \epsilon_t\|} + \frac{\|\nabla F(w_t)\|\|\hat \epsilon_t\|}{\|\nabla F(w_t) +\hat \epsilon_t\|}\\
    &\le -\frac{2}{3}\|\nabla F(w_t)\| + 2\|\hat \epsilon_t\|
\end{align*}
On the other hand, if $\|\hat \epsilon_t\| > \frac{1}{2}\|\nabla F(w_t)\|$, then
\begin{align*}
     -\left\langle\nabla F(w_t), \frac{\hat m_t}{\|\hat m_t\|} \right\rangle &\le 0\\
     &\le -\frac{2}{3}\|\nabla F(w_t)\| + \frac{2}{3}\|\nabla F(w_t)\|  \\
     &\le -\frac{2}{3}\|\nabla F(w_t)\| + \frac{4}{3}\|\hat \epsilon_t\|
\end{align*}
So either way, we have $ -\left\langle\nabla F(w_t), \frac{\hat m_t}{\|\hat m_t\|} \right\rangle \le -\frac{2}{3}\|\nabla F(w_t)\| + 2\|\hat \epsilon_t\|$ Now sum over t and rearrange to get:
\begin{align*}
    \frac{1}{T}\sum_{t=1}^T\|\nabla F(w_t)\| &\le \frac{3(F(w_1) - F(w_{T+1}))}{2\eta T} + \frac{3L\eta}{4} + \frac{3}{T}\sum_{t=1}^T\|\hat \epsilon_t\|
\end{align*}
Take expectation of both sides:
\begin{align*}
    \frac{1}{T}\E\left[\sum_{t=1}^T\|\nabla F(w_t)\|\right] &\le \frac{3\E\left[(F(w_1) - F(w_{T+1}))\right]}{2\eta T} + \frac{3L\eta}{4} + \frac{3}{T}\sum_{t=1}^T\E[\|\hat \epsilon_t\|]
\end{align*}
Pick $\hat w$ uniformly at random from $w_1,...,w_T$. Then:
\begin{align*}
    \E\left[\|\nabla F(\hat w)\|\right]&\le \frac{3\E\left[(F(w_1) - F(w_{T+1}))\right]}{2\eta T} + \frac{3L\eta}{4} + \frac{3}{T}\sum_{t=1}^T\E[\|\hat \epsilon_t\|]
\end{align*}
\end{proof}
\begin{Proposition}\label{prop:powersum}
Let $\alpha \in(0,1]$ and $N$ be an arbitrary positive integer. Then:
\begin{align*}
    \sum_{i=0}^\infty(1-\alpha)^{Ni}& \le \frac{1}{1-\exp(-1)}\max\left (1,\ \frac{1}{\alpha N}\right)
\end{align*}
\end{Proposition}
\begin{proof}
\begin{align*}
    \sum_{i=0}^\infty(1-\alpha)^{Ni}&=\sum_{i=0}^\infty\left((1-\alpha)^{N}\right)^i
    \intertext{using the identity $1-x\le \exp(-x)$:}
    &\le \sum_{i=0}^\infty\left(\exp(-\alpha N)\right)^i
\end{align*}
Now, if $\alpha N\le 1$, then we use the identity $\exp(-x)\le 1-x(1-\exp(-1))$ for  all $x\in[0,1]$:
\begin{align*}
    \sum_{i=0}^\infty(1-\alpha)^{Ni} &\le \sum_{i=0}^\infty\left[1-(1-\exp(-1))\alpha N  \right]^i\\
    &=\frac{1}{\alpha N (1-\exp(-1))}
\end{align*}

On the other hand, if $\alpha N >1$, then we have:
\begin{align*}
    \sum_{i=0}^\infty(1-\alpha)^{Ni}&\le \sum_{i=0}^\infty\left(\exp(-\alpha N)\right)^i\\
    &\le \sum_{i=0}^\infty\exp(-i)\\
    &\le \frac{1}{1-\exp(-1)}
\end{align*}
Putting the two cases together yields the desired result.
\end{proof}
\begin{Proposition}\label{prop:advpowersum}
Let $\gamma\in(0,1]$ $ \frac{1}{N}\ln \left(\frac{1}{1-\gamma}\right) \le \alpha \le1$ and let $N$ be an arbitrary positive integer. Then:
\begin{align*}
    \sum_{i=0}^\infty(1-\gamma)^{-i}(1-\alpha)^{Ni}& \le\frac{1}{1-\exp(-1)}\max\left (1,\ \frac{1}{\alpha N - \ln\left(\frac{1}{1-\gamma}\right)}\right)
\end{align*}
In particular, if $\gamma \le 1/2$ and $\frac{1+\ln(2)}{N}\le  \alpha$:
\begin{align*}
    \sum_{i=0}^\infty(1-\gamma)^{-i}(1-\alpha)^{Ni}&\le \frac{1}{1-\exp(-1)}
\end{align*}
\end{Proposition}
\begin{proof}
\begin{align*}
    \sum_{i=0}^\infty(1-\gamma)^{-i}(1-\alpha)^{Ni}&=\sum_{i=0}^\infty\left((1-\alpha)^{N}\right)^i
    \intertext{using the identity $1-x\le \exp(-x)$:}
    &\le \sum_{i=0}^\infty(1-\gamma)^i\left(\exp(-\alpha N)\right)^i\\
    &= \sum_{i=0}^\infty\left[\exp\left(-\alpha N + \ln\left(\frac{1}{1-\gamma}\right)\right)\right]^i
\end{align*}
Now, as in the proof of Proposition~\ref{prop:powersum}, we consider two cases. First, if $ 0\le \alpha N - \ln\left(\frac{1}{1-\gamma}\right)\le 1$, then since $\exp(-x)\le 1-(1-\exp(-1))x$ for all $x\in[0,1]$, we have:
\begin{align*}
    \sum_{i=0}^\infty(1-\gamma)^{-i}(1-\alpha)^{Ni}&\le \sum_{i=0}^\infty\left[1-(1-\exp(-1))\left(\alpha N - \ln\left(\frac{1}{1-\gamma}\right)\right)\right]^i\\
    &\le \frac{1}{(1-\exp(-1))\left(\alpha N - \ln\left(\frac{1}{1-\gamma}\right)\right)}
\end{align*}
Alternatively, if $\alpha N - \ln\left(\frac{1}{1-\gamma}\right)\ge 1$, 
\begin{align*}
    \sum_{i=0}^\infty(1-\gamma)^{-i}(1-\alpha)^{Ni}&\le\sum_{i=0}^\infty \exp(-i)\\
    &\le \frac{1}{1-\exp(-1)}
\end{align*}
Putting the two cases together provides the first statement in the Proposition.

For the second statement, observe that since $\ln\left(\frac{1}{1-\gamma}\right)$ is increasing in $\gamma$, we have $\frac{1+\ln(2)}{N}\le  \alpha$ implies $1+\ln\left(\frac{1}{1-\gamma}\right)\le N\alpha$, from which the result follows.
\end{proof}
\begin{Corollary}\label{cor:sumgamma}
Suppose $\gamma\in(0,1]$ and $1\ge \alpha \ge  \frac{1}{N}\ln\left(\frac{1}{1-\gamma}\right)$: Then:
\begin{align*}
    \sum_{j=0}^{i-1}(1-\gamma)^j(1-\alpha)^{(i-1-j)N}&\le  \frac{(1-\gamma)^{i-1}}{1-\exp(-1)}\max\left(1,\ \frac{1}{\alpha N - \ln\left(\frac{1}{1-\gamma}\right)}\right)
\end{align*}
\end{Corollary}
\begin{proof}
\begin{align*}
    \sum_{j=0}^{i-1}(1-\gamma)^j(1-\alpha)^{(i-1-j)N}&=(1-\gamma)^{i-1}\sum_{j=0}^{i-1}(1-\gamma)^{j-(i-1)}(1-\alpha)^{(i-1-j)N}\\
    &\le (1-\gamma)^{i-1}\sum_{k=0}^\infty (1-\gamma)^{-k}(1-\alpha)^{kN}
    \intertext{now apply Proposition~\ref{prop:advpowersum}:}
    &\le \frac{(1-\gamma)^{i-1}}{1-\exp(-1)}\max\left(1,\ \frac{1}{\alpha N - \ln\left(\frac{1}{1-\gamma}\right)}\right)
\end{align*}
\end{proof}
\begin{Proposition}\label{prop:binarysum}
Let $\gamma\in(0,1)$ and $s$ be any integer. Suppose $b_Rb_{R-1}\dots b_0$ is the binary expansion of $s$, so that $s=\sum_{i=0}^R b_i 2^i$. Then:
\begin{align*}
    \sum_{i=0}^R b_i \gamma^{2(s\mod 2^i)} \le 1+\sum_{i=1}^R \gamma^{2^i}
\end{align*}
\end{Proposition}
\begin{proof}

Define $F(s) = \sum_{i=0}^R b_i \gamma^{2(s\mod 2^i)}$ where $b_i$ is the $i$th bit of the binary expansion of $s$. Then we will show first:
\begin{align*}
    F\left(\sum_{i=0}^R 2^i\right) \le 1+\sum_{i=1}^R \gamma^{2^i}
\end{align*}
and then for all $s<2^{R+1}$,
\begin{align*}
    F(s) \le F\left(\sum_{i=0}^R 2^i\right)
\end{align*}
which together proves the Proposition.

For the first claim, we have:
\begin{align*}
\sum_{i=0}^R 2^i \mod 2^j &=2^{j}-1\text{ for all }j\le R+1
\intertext{so that:}
     F\left(\sum_{i=0}^R 2^i\right)&=\sum_{i=0}^R \gamma^{2(2^i-1)}\\
     &= 1+\sum_{i=1}^R \gamma^{2(2^i-1)}\\
     &\le 1+\sum_{i=1}^R \gamma^{2^i}
\end{align*}

Now, for the second claim, suppose that $s\ne \sum_{i=0}^R 2^i$. That is, there is some $j$ such that $b_j=0$. Equivalently, we can write $s= A+B$ such that $A<2^j$ and $B=0\mod 2^{j+1}$ for some $j\le R$. Then define $s' = A+B/2$. Let $b'_R\dots b'_0$ be the binary expansion of $s'$. Then we have $b'_i=b_i$ for $i<j$ and $b'_i = b_{i+1}$ for $i\ge j$. Further, for $i<j$, $s \mod 2^i=A\mod 2^i=s' \mod 2^i$, and for $i\ge j$, $s\mod 2^{i+1} = A + (B\mod 2^{i+1}) \ge A+(B/2 \mod 2^i)$. Thus, we have:
\begin{align*}
    F(s) = \sum_{i=0}^R b_i \gamma^{2(s\mod 2^i)}\le \sum_{i=0}^R b'_i \gamma^{2(s'\mod 2^i)} = F(s')
\end{align*}

By repeating this argument, we see that if $s$ has $n$ non-zero bits in the binary expansion, $F(s)\le F(2^n-1)$. Finally, notice that adding higher-order bits to the binary expansion can only increase $F$, so that $F(2^n-1)\le F(2^{R+1}-1)$ and so we are done.

\end{proof}
\begin{Lemma}\label{lem:breaksum}
Let $T=qN+r$ where $q\in \Z$ is the quotient and $r\in[0,N-1]$ is the remainder when dividing $T$ by $N$. Then:
\begin{align*}
    G_{[1,T]} &= (1-\alpha)^{T-r}G_{[1,r]}+\sum_{i=0}^{q-1}(1-\alpha)^{Ni}(1-\gamma)^{q-(i+1)}G_{[r+1,r+N]}\\
    &\quad\quad+\sum_{i=1}^{q-1}\left(\sum_{j=0}^{i-1} (1-\gamma)^j(1-\alpha)^{(i-1-j)N}\right)\left((1-\gamma)\Delta_{[T-iN+1,T-(i-1)N]} + \gamma G_{[T-iN+1,T-(i-1)N]}\right)
\end{align*}
\end{Lemma}
\begin{proof}
For any iteration t, let $t= q_tN+r_t$. We re-define:
\begin{align*}
    F^G_{[a,b]} &= \sum_{i=a}^b(1-\alpha)^{b-i}\nabla f(w_{q_iN}, x_{\pi^{q_i}_{r_i}})\\
     F^\Delta_{[a,b]} &= \sum_{i=a}^b(1-\alpha)^{b-i}(\nabla f(w_{q_iN}, x^{q_i}_{r_i}) - \nabla f((w_{(q_{i}-1)N}, x^{q_i}_{r_i}))
\end{align*}
and:
\begin{align*}
    G_{[a,b]} &= \sum_{[y,z] \in \COMPOSE(a,b)}(1-\alpha)^{b-z}F^G_{[y,z]}\\
    \Delta_{[a,b]} &= \sum_{[y,z] \in \COMPOSE(a,b)}(1-\alpha)^{b-z}F^\Delta_{[y,z]}
\end{align*}
Interpreting $\sum_{i=a}^b$ as 0 whenever $b<a$, we see that statement of the Lemma is immediate for $T<N$, as $q=0$, $r=T$ and all sums are zero. Next, a little calculation reveals the following identity:
\begin{align}
    G_{[1,T]} &= (1-\alpha)^N G_{[1,T-N]} + G_{[T-N+1,T]}\label{eqn:batchrecursion}
\end{align}
which implies the Lemma  for $q=1$ (since $T-N=r$).

Further, we have:
\begin{align*}
    G_{[T-N+1,T]} = (1-\gamma)\Delta_{[T-N+1,T]}  + \gamma G_{[T-N+1,T]} + (1-\gamma)G_{[T-2N+1,T-N]}
\end{align*}
Putting these together yields:
\begin{align*}
    G_{[1,T]} &= (1-\alpha)^N G_{[1,T-N]} + (1-\gamma)\Delta_{[T-N+1,T]}  + \gamma G_{[T-N+1,T]} + (1-\gamma)G_{[T-2N+1,T-N]}\\
    &= (1-\alpha)^{2N} G_{[1,T-2N]}+(1-\alpha)^NG_{[T-2N+1,T-N]} + (1-\gamma)\Delta_{[T-N+1,T]}  + \gamma G_{[t-N+1,T]} \\
    &\qquad+ (1-\gamma)G_{[T-2N+1,T-N]}\\
    &=(1-\alpha)^{2N}(G_{[1,T-2N]} + ((1-\alpha)^N+(1-\gamma))G_{[T-2N+1,T-N]} + (1-\gamma)\Delta_{[T-N+1,T]} + \gamma G_{[T-N+1,T]}
\end{align*}
which is exactly the statement of the Lemma for $q=2$ (since in this case $T-2N=r$).

Now, we proceed by induction on $q$: Suppose the statement holds for $(q-1)N+r$. That is, suppose
\begin{align*}
    G_{[1,T-N]} &= (1-\alpha)^{T-N-r}G_{[1,r]}+\sum_{i=0}^{q-2}(1-\alpha)^{iN}(1-\gamma)^{q-(i+1)}G_{[r+1,r+N]}\\
    &\quad\quad+\sum_{i=1}^{q-2}\left(\sum_{j=0}^{i-1} (1-\gamma)^j(1-\alpha)^{(i-1-j)N}\right)\left((1-\gamma)\Delta_{[T-(i+1)N+1,T-iN]} + \gamma G_{[T-(i+1)N+1,T-iN]}\right)
\end{align*}
Then, we observe the following identity if  $i<q$:
\begin{align*}
    G_{[T-iN+1,T-(i-1)N]} &=  (1-\gamma)\Delta_{[T-iN+1,T-(i-1)N]} + \gamma G_{[T-iN+1,T-(i-1)N]} + (1-\gamma)G_{[T-(i+1)N+1,T-iN]}
\end{align*}
From this we can conclude:
\begin{align*}
    G_{[T-iN+1,T-(i-1)N]} &= (1-\gamma)^{q-i}G_{[T-qN+1,T-(q-1)N]}\\
    &\quad\quad+\sum_{j=0}^{q-i-1} (1-\gamma)^{j}((1-\gamma)\Delta_{[T-(j+i)N+1,T-(j+i-1)N]} + \gamma G_{[T-(j+i)N+1,T-(j+i-1)N]})
\end{align*}

Now, put this together with (\ref{eqn:batchrecursion}):
\begin{align*}
    G_{[1,T]} &= (1-\alpha)^N G_{[1,T-N]} +(1-\gamma)^{q-1}G_{[T-qN+1,T-(q-1)N]}\\
    &\quad\quad+\sum_{j=0}^{q-2} (1-\gamma)^{j}((1-\gamma)\Delta_{[T-(j+1)N+1,T-jN]} + \gamma G_{[T-(j+1)N+1,T-jN]})
    \intertext{using the definition of $q,r$:}
    &= (1-\alpha)^N G_{[1,T-N]} +(1-\gamma)^{q-1}G_{[r+1,r+N]}\\
    &\quad\quad+\sum_{j=0}^{q-2} (1-\gamma)^{j}((1-\gamma)\Delta_{[T-(j+1)N+1,T-jN]} + \gamma G_{[T-(j+1)N+1,T-jN]})
    \intertext{using the induction hypothesis:}
    &= (1-\alpha)^{T-r}G_{[1,r]}+\sum_{i=0}^{q-2}(1-\alpha)^{(i+1)N}(1-\gamma)^{q-(i+2)}G_{[r+1,r+N]}\\
    &\quad\quad+\sum_{i=1}^{q-2}\left(\sum_{j=0}^{i-1} (1-\gamma)^j(1-\alpha)^{(i-j)N}\right)\left((1-\gamma)\Delta_{[T-(i+1)N+1,T-iN]} + \gamma G_{[T-(i+1)N+1,T-iN]}\right)\\ &\quad+(1-\gamma)^{q-1}G_{[r+1,r+N]}\\
    &\quad\quad+\sum_{j=0}^{q-2} (1-\gamma)^{j}((1-\gamma)\Delta_{[T-(j+1)N+1,T-jN]} + \gamma G_{[T-(j+1)N+1,T-jN]})
    \intertext{reindexing and combining the third line with the sum on the first line:}
    &= (1-\alpha)^{T-r}G_{[1,r]}+\sum_{i=0}^{q-1}(1-\alpha)^{iN}(1-\gamma)^{q-(i+1)}G_{[r+1,r+N]}\\
    &\quad\quad+\sum_{i=1}^{q-2}\left(\sum_{j=0}^{i-1} (1-\gamma)^j(1-\alpha)^{(i-j)N}\right)\left((1-\gamma)\Delta_{[T-(i+1)N+1,T-iN]} + \gamma G_{[T-(i+1)N+1,T-iN]}\right)\\
    &\quad\quad+\sum_{j=0}^{q-2} (1-\gamma)^{j}((1-\gamma)\Delta_{[T-(j+1)N+1,T-jN]} + \gamma G_{[T-(j+1)N+1,T-jN]})
    \intertext{reindexing:}
    &= (1-\alpha)^{T-r}G_{[1,r]}+\sum_{i=0}^{q-1}(1-\alpha)^{iN}(1-\gamma)^{q-(i+1)}G_{[r+1,r+N]}\\
    &\quad\quad+\sum_{i=2}^{q-1}\left(\sum_{j=0}^{i-2} (1-\gamma)^j(1-\alpha)^{(i-1-j)N}\right)\left((1-\gamma)\Delta_{[T-iN+1,T-(i-1)N]} + \gamma G_{[T-iN+1,T-(i-1)N]}\right)\\
    &\quad\quad+\sum_{j=0}^{q-2} (1-\gamma)^{j}((1-\gamma)\Delta_{[T-(j+1)N+1,T-jN]} + \gamma G_{[T-(j+1)N+1,T-jN]})\\
    &= (1-\alpha)^{T-r}G_{[1,r]}+\sum_{i=0}^{q-1}(1-\alpha)^{iN}(1-\gamma)^{q-(i+1)}G_{[r+1,r+N]}\\
    &\quad\quad+\sum_{i=2}^{q-1}\left(\sum_{j=0}^{i-2} (1-\gamma)^j(1-\alpha)^{(i-1-j)N}\right)\left((1-\gamma)\Delta_{[T-iN+1,T-(i-1)N]} + \gamma G_{[T-iN+1,T-(i-1)N]}\right)\\
    &\quad\quad+\sum_{i=1}^{q-1} (1-\gamma)^{i-1}((1-\gamma)\Delta_{[T-iN+1,T-(i-1)N]} + \gamma G_{[T-iN+1,T-(i-1)N]})\\
    &= (1-\alpha)^{T-r}G_{[1,r]}+\sum_{i=0}^{q-1}(1-\alpha)^{iN}(1-\gamma)^{q-(i+1)}G_{[r+1,r+N]}\\
    &\quad\quad+\sum_{i=2}^{q-1}\left(\sum_{j=0}^{i-1} (1-\gamma)^j(1-\alpha)^{(i-1-j)N}\right)\left((1-\gamma)\Delta_{[T-iN+1,T-(i-1)N]} + \gamma G_{[T-iN+1,T-(i-1)N]}\right)
\end{align*}
which establishes the claim.
\end{proof}
\begin{restatable}{Proposition}{propcomposebound}[Essentially \cite{daniely2015strongly}, Lemma 5]\label{prop:composebound}
The output of $S$ of $\COMPOSE(a,b)$ satisfies $|S|\le 2(1+\lfloor \log_2(b-a+1)\rfloor)$. In the special case that $a=1$, $|S|\le 1+\log_2(b)$. Moreover, each element of $\COMPOSE(a,b)$ is an interval of the form $[q2^k +1, (q+1)2^k]$ for some $q,k\in \mathbb{N}$, the intervals in $S$ are disjoint, and $[a,b] = \bigcup_{[x,y]\in S}[x,y]$.
\end{restatable}

%% file: neurips_2022.bbl
\begin{thebibliography}{46}
\providecommand{\natexlab}[1]{#1}
\providecommand{\url}[1]{\texttt{#1}}
\expandafter\ifx\csname urlstyle\endcsname\relax
  \providecommand{\doi}[1]{doi: #1}\else
  \providecommand{\doi}{doi: \begingroup \urlstyle{rm}\Url}\fi

\bibitem[Abadi et~al.(2016)Abadi, Chu, Goodfellow, McMahan, Mironov, Talwar,
  and Zhang]{abadi2016deep}
Martin Abadi, Andy Chu, Ian Goodfellow, H~Brendan McMahan, Ilya Mironov, Kunal
  Talwar, and Li~Zhang.
\newblock Deep learning with differential privacy.
\newblock In \emph{Proceedings of the 2016 ACM SIGSAC conference on computer
  and communications security}, pages 308--318, 2016.

\bibitem[Arjevani et~al.(2019)Arjevani, Carmon, Duchi, Foster, Srebro, and
  Woodworth]{arjevani2019lower}
Yossi Arjevani, Yair Carmon, John~C Duchi, Dylan~J Foster, Nathan Srebro, and
  Blake Woodworth.
\newblock Lower bounds for non-convex stochastic optimization.
\newblock \emph{arXiv preprint arXiv:1912.02365}, 2019.

\bibitem[Arjevani et~al.(2020)Arjevani, Carmon, Duchi, Foster, Sekhari, and
  Sridharan]{arjevani2020second}
Yossi Arjevani, Yair Carmon, John~C Duchi, Dylan~J Foster, Ayush Sekhari, and
  Karthik Sridharan.
\newblock Second-order information in non-convex stochastic optimization: Power
  and limitations.
\newblock In \emph{Conference on Learning Theory}, pages 242--299, 2020.

\bibitem[Asi et~al.(2021)Asi, Duchi, Fallah, Javidbakht, and
  Talwar]{asi2021private}
Hilal Asi, John Duchi, Alireza Fallah, Omid Javidbakht, and Kunal Talwar.
\newblock Private adaptive gradient methods for convex optimization.
\newblock In \emph{International Conference on Machine Learning}, pages
  383--392. PMLR, 2021.

\bibitem[Backstrom et~al.(2007)Backstrom, Dwork, and
  Kleinberg]{10.1145/1242572.1242598}
Lars Backstrom, Cynthia Dwork, and Jon Kleinberg.
\newblock Wherefore art thou r3579x? anonymized social networks, hidden
  patterns, and structural steganography.
\newblock In \emph{Proceedings of the 16th International Conference on World
  Wide Web}, WWW '07, page 181–190, New York, NY, USA, 2007. Association for
  Computing Machinery.
\newblock ISBN 9781595936547.
\newblock \doi{10.1145/1242572.1242598}.
\newblock URL \url{https://doi.org/10.1145/1242572.1242598}.

\bibitem[Balle et~al.(2018)Balle, Barthe, and Gaboardi]{balle2018privacy}
Borja Balle, Gilles Barthe, and Marco Gaboardi.
\newblock Privacy amplification by subsampling: Tight analyses via couplings
  and divergences.
\newblock \emph{Advances in Neural Information Processing Systems}, 31, 2018.

\bibitem[Bassily et~al.(2014)Bassily, Smith, and Thakurta]{bassily2014private}
Raef Bassily, Adam Smith, and Abhradeep Thakurta.
\newblock Private empirical risk minimization: Efficient algorithms and tight
  error bounds.
\newblock In \emph{2014 IEEE 55th Annual Symposium on Foundations of Computer
  Science}, pages 464--473. IEEE, 2014.

\bibitem[Bassily et~al.(2019)Bassily, Feldman, Talwar, and
  Thakurta]{bassily2019private}
Raef Bassily, Vitaly Feldman, Kunal Talwar, and Abhradeep Thakurta.
\newblock Private stochastic convex optimization with optimal rates.
\newblock In \emph{Proceedings of the 33rd International Conference on Neural
  Information Processing Systems}, pages 11282--11291, 2019.

\bibitem[Bassily et~al.(2021{\natexlab{a}})Bassily, Guzm{\'a}n, and
  Menart]{bassily2021differentially}
Raef Bassily, Crist{\'o}bal Guzm{\'a}n, and Michael Menart.
\newblock Differentially private stochastic optimization: New results in convex
  and non-convex settings.
\newblock \emph{Advances in Neural Information Processing Systems}, 34,
  2021{\natexlab{a}}.

\bibitem[Bassily et~al.(2021{\natexlab{b}})Bassily, Guzm{\'a}n, and
  Nandi]{bassily2021non}
Raef Bassily, Crist{\'o}bal Guzm{\'a}n, and Anupama Nandi.
\newblock Non-euclidean differentially private stochastic convex optimization.
\newblock In \emph{Conference on Learning Theory}, pages 474--499. PMLR,
  2021{\natexlab{b}}.

\bibitem[Chan et~al.(2011)Chan, Shi, and Song]{chan2011private}
T-H~Hubert Chan, Elaine Shi, and Dawn Song.
\newblock Private and continual release of statistics.
\newblock \emph{ACM Transactions on Information and System Security (TISSEC)},
  14\penalty0 (3):\penalty0 1--24, 2011.

\bibitem[Chaudhuri et~al.(2011)Chaudhuri, Monteleoni, and
  Sarwate]{chaudhuri2011differentially}
Kamalika Chaudhuri, Claire Monteleoni, and Anand~D Sarwate.
\newblock Differentially private empirical risk minimization.
\newblock \emph{Journal of Machine Learning Research}, 12\penalty0 (3), 2011.

\bibitem[Chen et~al.(2021)Chen, Liu, Hogan, Shenkman, and
  Bian]{chen2021applications}
Zhaoyi Chen, Xiong Liu, William Hogan, Elizabeth Shenkman, and Jiang Bian.
\newblock Applications of artificial intelligence in drug development using
  real-world data.
\newblock \emph{Drug discovery today}, 26\penalty0 (5):\penalty0 1256--1264,
  2021.

\bibitem[Cutkosky and Mehta(2020)]{cutkosky2020momentum}
Ashok Cutkosky and Harsh Mehta.
\newblock Momentum improves normalized sgd.
\newblock In \emph{International Conference on Machine Learning}, 2020.

\bibitem[Cutkosky and Orabona(2019)]{cutkosky2019momentum}
Ashok Cutkosky and Francesco Orabona.
\newblock Momentum-based variance reduction in non-convex sgd.
\newblock In \emph{Advances in Neural Information Processing Systems}, pages
  15210--15219, 2019.

\bibitem[Daniely et~al.(2015)Daniely, Gonen, and
  Shalev-Shwartz]{daniely2015strongly}
Amit Daniely, Alon Gonen, and Shai Shalev-Shwartz.
\newblock Strongly adaptive online learning.
\newblock In \emph{International Conference on Machine Learning}, pages
  1405--1411. PMLR, 2015.

\bibitem[Dinur and Nissim(2003)]{10.1145/773153.773173}
Irit Dinur and Kobbi Nissim.
\newblock Revealing information while preserving privacy.
\newblock In \emph{Proceedings of the Twenty-Second ACM SIGMOD-SIGACT-SIGART
  Symposium on Principles of Database Systems}, PODS '03, page 202–210, New
  York, NY, USA, 2003. Association for Computing Machinery.
\newblock ISBN 1581136706.
\newblock \doi{10.1145/773153.773173}.
\newblock URL \url{https://doi.org/10.1145/773153.773173}.

\bibitem[Dwork and Roth(2014)]{dwork2014algorithmic}
Cynthia Dwork and Aaron Roth.
\newblock The algorithmic foundations of differential privacy.
\newblock \emph{Found. Trends Theor. Comput. Sci.}, 9\penalty0 (3-4):\penalty0
  211--407, 2014.

\bibitem[Dwork et~al.(2006)Dwork, McSherry, Nissim, and
  Smith]{10.1007/11681878_14}
Cynthia Dwork, Frank McSherry, Kobbi Nissim, and Adam Smith.
\newblock Calibrating noise to sensitivity in private data analysis.
\newblock In \emph{Proceedings of the Third Conference on Theory of
  Cryptography}, TCC'06, page 265–284, Berlin, Heidelberg, 2006.
  Springer-Verlag.
\newblock ISBN 3540327312.
\newblock \doi{10.1007/11681878_14}.
\newblock URL \url{https://doi.org/10.1007/11681878_14}.

\bibitem[Dwork et~al.(2010)Dwork, Naor, Pitassi, and
  Rothblum]{dwork2010differential}
Cynthia Dwork, Moni Naor, Toniann Pitassi, and Guy~N Rothblum.
\newblock Differential privacy under continual observation.
\newblock In \emph{Proceedings of the forty-second ACM symposium on Theory of
  computing}, pages 715--724, 2010.

\bibitem[Fang et~al.(2018)Fang, Li, Lin, and Zhang]{fang2018spider}
Cong Fang, Chris~Junchi Li, Zhouchen Lin, and Tong Zhang.
\newblock Spider: Near-optimal non-convex optimization via stochastic
  path-integrated differential estimator.
\newblock In \emph{Advances in Neural Information Processing Systems}, pages
  689--699, 2018.

\bibitem[Feldman et~al.(2018)Feldman, Mironov, Talwar, and
  Thakurta]{feldman2018privacy}
Vitaly Feldman, Ilya Mironov, Kunal Talwar, and Abhradeep Thakurta.
\newblock Privacy amplification by iteration.
\newblock In \emph{2018 IEEE 59th Annual Symposium on Foundations of Computer
  Science (FOCS)}, pages 521--532. IEEE, 2018.

\bibitem[Feldman et~al.(2020)Feldman, Koren, and Talwar]{feldman2020private}
Vitaly Feldman, Tomer Koren, and Kunal Talwar.
\newblock Private stochastic convex optimization: optimal rates in linear time.
\newblock In \emph{Proceedings of the 52nd Annual ACM SIGACT Symposium on
  Theory of Computing}, pages 439--449, 2020.

\bibitem[Ghadimi and Lan(2013)]{ghadimi2013stochastic}
Saeed Ghadimi and Guanghui Lan.
\newblock Stochastic first-and zeroth-order methods for nonconvex stochastic
  programming.
\newblock \emph{SIAM Journal on Optimization}, 23\penalty0 (4):\penalty0
  2341--2368, 2013.

\bibitem[Guha~Thakurta and Smith(2013)]{guha2013nearly}
Abhradeep Guha~Thakurta and Adam Smith.
\newblock (nearly) optimal algorithms for private online learning in
  full-information and bandit settings.
\newblock \emph{Advances in Neural Information Processing Systems},
  26:\penalty0 2733--2741, 2013.

\bibitem[He et~al.(2019)He, Baxter, Xu, Xu, Zhou, and Zhang]{article}
Jianxing He, Sally Baxter, Jie Xu, Jiming Xu, Xingtao Zhou, and Kang Zhang.
\newblock The practical implementation of artificial intelligence technologies
  in medicine.
\newblock \emph{Nature Medicine}, 25, 01 2019.
\newblock \doi{10.1038/s41591-018-0307-0}.

\bibitem[Iyengar et~al.(2019)Iyengar, Near, Song, Thakkar, Thakurta, and
  Wang]{8835258}
Roger Iyengar, Joseph~P. Near, Dawn Song, Om~Thakkar, Abhradeep Thakurta, and
  Lun Wang.
\newblock Towards practical differentially private convex optimization.
\newblock In \emph{2019 IEEE Symposium on Security and Privacy (SP)}, pages
  299--316, 2019.
\newblock \doi{10.1109/SP.2019.00001}.

\bibitem[Jayaraman et~al.(2018)Jayaraman, Wang, Evans, and
  Gu]{NEURIPS2018_7221e5c8}
Bargav Jayaraman, Lingxiao Wang, David Evans, and Quanquan Gu.
\newblock Distributed learning without distress: Privacy-preserving empirical
  risk minimization.
\newblock In S.~Bengio, H.~Wallach, H.~Larochelle, K.~Grauman, N.~Cesa-Bianchi,
  and R.~Garnett, editors, \emph{Advances in Neural Information Processing
  Systems}, volume~31. Curran Associates, Inc., 2018.
\newblock URL
  \url{https://proceedings.neurips.cc/paper/2018/file/7221e5c8ec6b08ef6d3f9ff3ce6eb1d1-Paper.pdf}.

\bibitem[Jiang et~al.(2017)Jiang, Jiang, Zhi, Dong, Li, Ma, Wang, Dong, Shen,
  and Wang]{Jiang230}
Fei Jiang, Yong Jiang, Hui Zhi, Yi~Dong, Hao Li, Sufeng Ma, Yilong Wang, Qiang
  Dong, Haipeng Shen, and Yongjun Wang.
\newblock Artificial intelligence in healthcare: past, present and future.
\newblock \emph{Stroke and Vascular Neurology}, 2\penalty0 (4):\penalty0
  230--243, 2017.
\newblock ISSN 2059-8688.
\newblock \doi{10.1136/svn-2017-000101}.
\newblock URL \url{https://svn.bmj.com/content/2/4/230}.

\bibitem[Kairouz et~al.(2021)Kairouz, McMahan, Song, Thakkar, Thakurta, and
  Xu]{kairouz2021practical}
Peter Kairouz, Brendan McMahan, Shuang Song, Om~Thakkar, Abhradeep Thakurta,
  and Zheng Xu.
\newblock Practical and private (deep) learning without sampling or shuffling.
\newblock \emph{arXiv preprint arXiv:2103.00039}, 2021.

\bibitem[Kifer et~al.(2012)Kifer, Smith, and Thakurta]{pmlr-v23-kifer12}
Daniel Kifer, Adam Smith, and Abhradeep Thakurta.
\newblock Private convex empirical risk minimization and high-dimensional
  regression.
\newblock In Shie Mannor, Nathan Srebro, and Robert~C. Williamson, editors,
  \emph{Proceedings of the 25th Annual Conference on Learning Theory},
  volume~23 of \emph{Proceedings of Machine Learning Research}, pages
  25.1--25.40, Edinburgh, Scotland, 25--27 Jun 2012. PMLR.
\newblock URL \url{https://proceedings.mlr.press/v23/kifer12.html}.

\bibitem[Kulkarni et~al.(2021)Kulkarni, Lee, and Liu]{kulkarni2021private}
Janardhan Kulkarni, Yin~Tat Lee, and Daogao Liu.
\newblock Private non-smooth empirical risk minimization and stochastic convex
  optimization in subquadratic steps.
\newblock \emph{arXiv preprint arXiv:2103.15352}, 2021.

\bibitem[Mironov(2017)]{mironov2017renyi}
Ilya Mironov.
\newblock R{\'e}nyi differential privacy.
\newblock In \emph{2017 IEEE 30th Computer Security Foundations Symposium
  (CSF)}, pages 263--275. IEEE, 2017.

\bibitem[Mishchenko et~al.(2020)Mishchenko, Khaled, and
  Richt{\'a}rik]{mishchenko2020random}
Konstantin Mishchenko, Ahmed Khaled, and Peter Richt{\'a}rik.
\newblock Random reshuffling: Simple analysis with vast improvements.
\newblock \emph{arXiv preprint arXiv:2006.05988}, 2020.

\bibitem[Nguyen et~al.(2019)Nguyen, Nguyen, and van Dijk]{nguyen2019tight}
Phuong~Ha Nguyen, Lam~M Nguyen, and Marten van Dijk.
\newblock Tight dimension independent lower bound on the expected convergence
  rate for diminishing step sizes in sgd.
\newblock \emph{Advances in Neural Information Processing Systems}, 32, 2019.

\bibitem[Talwar et~al.(2014)Talwar, Thakurta, and Zhang]{talwar2014private}
Kunal Talwar, Abhradeep Thakurta, and Li~Zhang.
\newblock Private empirical risk minimization beyond the worst case: The effect
  of the constraint set geometry.
\newblock \emph{arXiv preprint arXiv:1411.5417}, 2014.

\bibitem[Tran and Cutkosky(2021)]{tran2021better}
Hoang Tran and Ashok Cutkosky.
\newblock Better sgd using second-order momentum.
\newblock \emph{arXiv preprint arXiv:2103.03265}, 2021.

\bibitem[Tran-Dinh et~al.(2019)Tran-Dinh, Pham, Phan, and
  Nguyen]{tran2019hybrid}
Quoc Tran-Dinh, Nhan~H Pham, Dzung~T Phan, and Lam~M Nguyen.
\newblock A hybrid stochastic optimization framework for stochastic composite
  nonconvex optimization.
\newblock \emph{arXiv preprint arXiv:1907.03793}, 2019.

\bibitem[Wang et~al.(2017)Wang, Ye, and Xu]{wang2017differentially}
Di~Wang, Minwei Ye, and Jinhui Xu.
\newblock Differentially private empirical risk minimization revisited: faster
  and more general.
\newblock In \emph{Proceedings of the 31st International Conference on Neural
  Information Processing Systems}, pages 2719--2728, 2017.

\bibitem[Wang et~al.(2018)Wang, Ye, and Xu]{wang2018differentially}
Di~Wang, Minwei Ye, and Jinhui Xu.
\newblock Differentially private empirical risk minimization revisited: Faster
  and more general.
\newblock \emph{arXiv preprint arXiv:1802.05251}, 2018.

\bibitem[Wang et~al.(2019{\natexlab{a}})Wang, Chen, and
  Xu]{wang2019differentially}
Di~Wang, Changyou Chen, and Jinhui Xu.
\newblock Differentially private empirical risk minimization with non-convex
  loss functions.
\newblock In \emph{International Conference on Machine Learning}, pages
  6526--6535. PMLR, 2019{\natexlab{a}}.

\bibitem[Wang et~al.(2019{\natexlab{b}})Wang, Jayaraman, Evans, and
  Gu]{wang2019efficient}
Lingxiao Wang, Bargav Jayaraman, David Evans, and Quanquan Gu.
\newblock Efficient privacy-preserving stochastic nonconvex optimization.
\newblock \emph{arXiv preprint arXiv:1910.13659}, 2019{\natexlab{b}}.

\bibitem[Wu et~al.(2017)Wu, Li, Kumar, Chaudhuri, Jha, and
  Naughton]{wu2017bolt}
Xi~Wu, Fengan Li, Arun Kumar, Kamalika Chaudhuri, Somesh Jha, and Jeffrey
  Naughton.
\newblock Bolt-on differential privacy for scalable stochastic gradient
  descent-based analytics.
\newblock In \emph{Proceedings of the 2017 ACM International Conference on
  Management of Data}, pages 1307--1322, 2017.

\bibitem[Zhang et~al.(2017)Zhang, Zheng, Mou, and Wang]{zhang2017efficient}
Jiaqi Zhang, Kai Zheng, Wenlong Mou, and Liwei Wang.
\newblock Efficient private erm for smooth objectives.
\newblock In \emph{Proceedings of the 26th International Joint Conference on
  Artificial Intelligence}, pages 3922--3928, 2017.

\bibitem[Zhou et~al.(2018)Zhou, Xu, and Gu]{zhou2018stochastic}
Dongruo Zhou, Pan Xu, and Quanquan Gu.
\newblock Stochastic nested variance reduction for nonconvex optimization.
\newblock In \emph{Proceedings of the 32nd International Conference on Neural
  Information Processing Systems}, pages 3925--3936, 2018.

\bibitem[Zhou et~al.(2020)Zhou, Chen, Hong, Wu, and Banerjee]{zhou2020private}
Yingxue Zhou, Xiangyi Chen, Mingyi Hong, Zhiwei~Steven Wu, and Arindam
  Banerjee.
\newblock Private stochastic non-convex optimization: Adaptive algorithms and
  tighter generalization bounds.
\newblock \emph{arXiv preprint arXiv:2006.13501}, 2020.

\end{thebibliography}
